\documentclass{article}

\usepackage[final,nonatbib]{neurips_2020}

\usepackage{microtype}
\usepackage{booktabs} 

\usepackage[utf8]{inputenc}
\usepackage[T1]{fontenc}
\usepackage{url}
\usepackage{graphicx,wrapfig,lipsum}
\usepackage{amsfonts}
\usepackage{nicefrac}
\usepackage{microtype}
\usepackage{dsfont}
\usepackage{subcaption}
\usepackage{MnSymbol}
\usepackage{algorithm}
\usepackage[noend]{algorithmic}
\usepackage{booktabs,tabularx, multirow}
\usepackage{tikz}
\usetikzlibrary{bayesnet}
\usepackage{mathtools}
\usepackage{amsmath,amsthm}
\usepackage{soul}
\usepackage{blindtext}
\usepackage{tcolorbox}
\usepackage[hidelinks]{hyperref}
\usepackage{todonotes}
\usepackage{thmtools}

\DeclareRobustCommand{\rchi}{{\mathpalette\irchi\relax}}
\newcommand{\irchi}[2]{\raisebox{\depth}{$#1\chi$}} 

\DeclarePairedDelimiterX{\infdivx}[2]{(}{)}{%
  #1\;\delimsize\|\;#2%
}
\DeclarePairedDelimiterX{\inftvx}[2]{(}{)}{%
  #1, #2%
}
\newcommand{\kl}{\Delta_\textrm{KL}\infdivx}
\newcommand{\tv}{\Delta_\textrm{TV}\inftvx}
\newcommand{\tvsq}{\Delta^2_\textrm{TV}\inftvx}
\newcommand{\chidiv}{\Delta_{\chi^2}\infdivx}

\DeclareMathOperator*{\argmin}{arg\,min}
\newcommand{\norm}[1]{\left\lVert#1\right\rVert}
\DeclareMathOperator{\Tr}{Tr}
\newtheorem{prop}{Proposition}

\newtheorem{lemma}{Lemma}

\usepackage{hyperref}

\title{Decision-Making with \\ Auto-Encoding Variational Bayes}

\author{%
Romain Lopez$^{1}$, Pierre Boyeau$^1$, Nir Yosef$^{1,2,3}$, Michael I. Jordan$^{1,4}$, and Jeffrey Regier$^5$\\ 
~\\
$^1$ Department of Electrical Engineering and Computer Sciences,\\ University of California, Berkeley\\
$^2$ Chan-Zuckerberg Biohub, San Francisco \\
$^3$ Ragon Institute of MGH, MIT and Harvard \\
$^4$ Department of Statistics, University of California, Berkeley\\
$^5$ Department of Statistics, University of Michigan
}

\newcommand\blfootnote[1]{%
  \begingroup
  \renewcommand\thefootnote{}\footnote{#1}%
  \addtocounter{footnote}{-1}%
  \endgroup
}

\begin{document}

\maketitle

\begin{abstract}
To make decisions based on a model fit with auto-encoding variational Bayes (AEVB), practitioners often let the variational distribution serve as a surrogate for the posterior distribution. This approach yields biased estimates of the expected risk, and therefore leads to poor decisions for two reasons. First, the model fit with AEVB may not equal the underlying data distribution. Second, the variational distribution may not equal the posterior distribution under the fitted model.
We explore how fitting the variational distribution based on several objective functions other than the ELBO, while continuing to fit the generative model based on the ELBO, affects the quality of downstream decisions.
For the probabilistic principal component analysis model, we investigate how importance sampling error, as well as the bias of the model parameter estimates, varies across several approximate posteriors when used as proposal distributions.
Our theoretical results suggest that a posterior approximation distinct from the variational distribution should be used for making decisions. Motivated by these theoretical results, we propose learning several approximate proposals for the best model and combining them using multiple importance sampling for decision-making. In addition to toy examples, we present a full-fledged case study of single-cell RNA sequencing. In this challenging instance of multiple hypothesis testing, our proposed approach surpasses the current state of the art.
\end{abstract}

\section{Introduction}

The auto-encoding variational Bayes (AEVB) algorithm performs model selection by maximizing a lower bound on the model evidence~\cite{AEVB, Rezende2014}. In the specific case of variational autoencoders (VAEs), a low-dimensional representation of data is transformed through a learned nonlinear function (another neural network) into the parameters of a conditional likelihood. VAEs achieve impressive performance on pattern-matching tasks like representation/manifold learning and synthetic image generation~\cite{Gulrajani2017}.

Many machine learning applications, however, require decisions, not just compact representations of the data. 
Researchers have accordingly attempted to use VAEs for decision-making applications, including novelty detection in control applications~\cite{Amini2018}, mutation-effect prediction for genomic sequences~\cite{Riesselman2018}, artifact detection~\cite{Ding2018}, and Bayesian hypothesis testing for single-cell RNA sequencing data~\cite{Lopez292037, Xu2019}.
To make decisions based on VAEs, these researchers implicitly appeal to Bayesian decision theory, which counsels taking the action that minimizes expected loss under the posterior distribution~\cite{Fienberg}.

However, for VAEs, the relevant functionals of the posterior cannot be computed exactly.
Instead, after fitting a VAE based on the ELBO, practitioners take one of three approaches to decision-making: i) the variational distribution may be used as a surrogate for the posterior~\cite{Riesselman2018}, ii) the variational distribution may be used as a proposal distribution for importance sampling~\cite{NIPS2018_7699}, or iii) the variational distribution can be ignored once the model is fit, and decisions may be based on an iterative sampling method such as MCMC or annealed importance sampling~\cite{wu2016quantitative}. But will any of these combined procedures (ELBO for model training and one of these methods for approximating posterior expectations) produce good decisions?

They may not, for two reasons.
First, estimates of the relevant expectations of the posterior may be biased and/or may have high variance. The former situation is typical when the variational distribution is substituted for the posterior; the latter is common for importance sampling estimators. By using the variational distribution as a proposal distribution, practitioners aim to get unbiased low-variance estimates of posterior expectations. But this approach often fails. The variational distribution recovered by the VAE, which minimizes the reverse Kullback-Leibler (KL) divergence between the variational distribution and the model posterior, is known to systematically underestimate variance~\cite{wainwright2008graphical,Turner2011}, making it a poor choice for an importance sampling proposal distribution. Alternative inference procedures have been proposed to address this problem. For example, expectation propagation (EP)~\cite{minka2013expectation} and CHIVI~\cite{NIPS2017_6866} minimize the forward KL divergence and the $\rchi^2$ divergence, respectively. Both objectives have favorable properties for fitting a proposal distribution~\cite{Chatterjee2018,Agapiou2017}. IWVI~\cite{NIPS2018_7699} seeks to maximize a tight lower bound of the evidence that is based on importance sampling estimates (IWELBO). Empirically, IWVI outperforms VI for estimating posterior expectations. It is unclear, however, which method to choose for a particular application.

Second, even if we can faithfully compute expectations of the model posterior, the model learned by the VAE may not resemble the real data-generating process~\cite{Turner2011}. Most VAE frameworks rely on the IWELBO, where the variational distribution is used as a proposal~\cite{BurdaGS15,rainforth2018tighter,chen2018variational}. For example, model and inference parameters are jointly learned in the IWAE~\cite{BurdaGS15} using the IWELBO. Similarly, the wake-wake (WW) procedure~\cite{Bornschein2015,le2018revisiting} uses the IWELBO for learning the model parameters but seeks to find a variational distribution that minimizes the forward KL divergence. In the remainder of this manuscript, we will use the same name to refer to either the inference procedure or the associated VAE framework (e.g., WW will be used to refer to EP). 

To address both of these issues, we propose a simple three-step procedure for making decisions with VAEs. First, we fit a model based on one of several objective functions (e.g., VAE, IWAE, WW, or $\rchi$-VAE) and select the best model based on some metric (e.g., IWELBO calculated on held-out data with a large numbers of particles). The $\rchi$-VAE is a novel variant of the WW algorithm that, for fixed $p_\theta$, minimizes the $\rchi^2$ divergence (for further details, see Appendix~\ref{app:chi-vaes}). 
Second, with the model fixed, we fit several approximate posteriors, based on the same objective functions, as well as annealed importance sampling~\cite{wu2016quantitative}. Third, we combine the approximate posteriors as proposal distributions for multiple importance sampling~\cite{veach1995optimally} to make decisions that minimize the expected loss under the posterior. 
In multiple importance sampling, we expect the mixture to be a better proposal than either of its components alone, especially in settings where the posterior is complex because each component can capture different parts of the posterior.

After introducing the necessary background (Section~\ref{sec:related}),
we provide a complete analysis of our framework for the probabilistic PCA model~\cite{Bishop:2006:PRM:1162264} (Section~\ref{sec:lin_VAE_analysis}). 
In this tractable setting, we recover the known fact that an underdispersed proposal causes severe error to importance sampling estimators~\cite{pmlr-v80-yao18a}. The analysis also shows that overdispersion may harm the process of model learning by exacerbating existing biases in variational Bayes. We also confirm these results empirically. Next, we perform an extensive empirical evaluation of two real-world decision-making problems. 
First, we consider a practical instance of classification-based decision
theory. 
In this setting, we show that the vanilla VAE becomes overconfident in its posterior predictive density, which harms performance. We also show that our three-step procedure outperforms IWAE and WW (Section~\ref{label_decision}).
We then present a scientific case study, focusing on an instance of multiple hypothesis testing in single-cell RNA sequencing data. Our approach yields a better calibrated estimate of the expected posterior false discovery rate (FDR) than that computed by the current state-of-the-art method (Section~\ref{FDR}). 

\blfootnote{Our code is available at
\url{http://github.com/PierreBoyeau/decision-making-vaes}}

\section{Background}
\label{sec:related}
Bayesian decision-making~\cite{Fienberg} makes use of a model and its posterior distribution to make optimal decisions. We bring together several lines of research in an overall Bayesian framework.

\subsection{Auto-encoding variational Bayes}
\label{aevb}
Variational autoencoders~\cite{AEVB} are based on a hierarchical Bayesian model~\cite{GelmanHill:2007}. Let $x$ be the observed random variables and $z$ the latent variables. To learn a generative model $p_\theta(x, z)$ that maximizes the evidence $\log p_\theta(x)$, variational Bayes~\cite{wainwright2008graphical} uses a proposal distribution $q_\phi(z \mid x)$ to approximate the posterior $p_\theta(z \mid x)$. The evidence decomposes as the \emph{evidence lower bound} (ELBO) and the reverse KL variational gap (VG):
\begin{align}
    \log p_\theta(x) &= \mathbb{E}_{q_\phi(z \mid x)}\log \frac{p_\theta(x, z)}{q_\phi(z \mid x)} + \kl{q_\phi}{p_\theta}. \label{ELBO}
\end{align}
Here we adopt the condensed notation $\kl{q_\phi}{p_\theta}$ to refer to the KL divergence between $q_\phi(z\mid x)$ and $p_\theta(z \mid x)$.
In light of this decomposition, a valid inference procedure involves jointly maximizing the ELBO with respect to the model's parameters and the variational distribution. The resulting variational distribution minimizes the reverse KL divergence. VAEs parameterize the variational distribution with a neural network. Stochastic gradients of the ELBO with respect to the variational parameters are computed via the reparameterization trick~\cite{AEVB}. 

\subsection{Approximation of posterior expectations}
\label{ss:posterior_approx}
Given a model $p_\theta$, an action set $\mathcal{A}$, and a loss $L$, the optimal decision $a^*(x)$ for observation $x$ is an expectation taken with respect to the posterior:
\(
    \mathcal{Q}(f, x) = \mathbb{E}_{p_\theta(z \mid x)}f(z).
\)
Here $f$ depends on the loss~\cite{Fienberg}. We therefore focus on numerical methods for estimating $\mathcal{Q}(f, x)$. Evaluating these expectations is the aim of Markov chain Monte Carlo (MCMC), annealed importance sampling (AIS)~\cite{neal2001annealed}, and variational methods~\cite{NIPS2018_7699}. 

Although we typically lack direct access to the posterior $p_\theta(z \mid x)$, we can, however, sample $(z_i)_{1 \leq i \leq n}$ from the variational distribution $q_\phi(z \mid x)$.  A naive but practical approach is to consider a plugin estimator~\cite{Amini2018, Riesselman2018, Ding2018, Lopez292037}:
\begin{align}
\hat{\mathcal{Q}}^n_{\textrm{P}}(f, x) = \frac{1}{n}\sum_{i=1}^nf(z_i).
\end{align}
This estimator replaces the exact posterior by sampling $z_1,\ldots,z_n$ from $q_\phi(z \mid x)$.  A less naive approach is to use
self-normalized importance sampling (SNIS):
\begin{align}
\hat{\mathcal{Q}}^n_{\textrm{IS}}(f, x) = 
    \frac{\sum_{i=1}^nw(x, z_i)f(z_i)}{\sum_{j=1}^nw(x, z_j)}.
\end{align}
Here the importance weights are
\(
    w(x, z) := \nicefrac{p_\theta(x, z)}{q_\phi(z\mid x)}. 
\)
The non-asymptotic behavior of both estimators and their variants is well understood~\cite{Chatterjee2018,Agapiou2017,Cortes2010}. Moreover, each upper bound on the error motivates an alternative inference procedure in which the upper bound is used as a surrogate for the error. For example, \cite{Chatterjee2018} bounds the error of the IS estimator with a function of \textit{forward} KL divergence $\kl{p_\theta}{q_\phi}$, which motivates the WW algorithm in~\cite{le2018revisiting}. Similarly,~\cite{Agapiou2017} provides an upper bound of the error based on the $\rchi^2$ divergence, which motivates our investigation of the $\rchi$-VAE. However, these upper bounds are too loose to explicitly compare, for example, the worst-case performance of $\rchi$-VAE and WW when $f$ belongs to a function class (for further details, see Appendix~\ref{app:err_simple_bounds}).

\section{Theoretical analysis for pPCA}
\label{sec:lin_VAE_analysis}
We aim to understand the theoretical advantages and disadvantages of each objective function used for training VAEs and how they impact downstream decision-making procedures.
Because intractability prevents us from deriving sharp constants in general models, to support a precise theoretical analysis, we consider probabilistic principal component analysis (pPCA)~\cite{tipping1999probabilistic}. pPCA is a linear model for which posterior inference is tractable. 
Though the analysis is a special case, we believe that it provides an intuition for the performance of our decision-making procedures more generally because, for many of the models that are used in practice, a Gaussian distribution approximates the posterior well, as demonstrated by the success of the Laplace approximation~\cite{laplace1986}.

In pPCA, the latent variables $z$ generate data $x$. We use an isotropic Gaussian prior on $z$ and a spherical Gaussian likelihood for $x$:
\begin{align}
\label{eq:lin_vae_gen}
\begin{split}
p_\theta(z) &=\textrm{Normal}(0, I) \\
p_\theta(x \mid z) &= \textrm{Normal}(Wz + \mu, \sigma^2I).
\end{split}
\end{align}
Following~\cite{Turner2011}, we parameterize $\sigma^2 = \nicefrac{1}{(1 - \lambda^2)}$, as well as $W_{ij} = e^\lambda W'_{ij}$ if $i \neq j$ and $W_{ij} = W'_{ij}$ otherwise. This parameterization is designed to make model selection challenging. Here $\theta := (W', \mu, \lambda)$. We consider a class of amortized posterior approximations with the following form:
\begin{align}
\label{eq:lin_vae_inf}
q_\phi(z \mid x) &= \textrm{Normal}\left(h_\eta(x), D(x)\right).
\end{align}
Here $h_\eta$ is a neural network with parameters $\eta$, $D(x)$ is the diagonal covariance matrix given by $\mathrm{diag}(h_\xi(x))$, and $h_\xi$ is a neural network with parameters $\xi$. In this example, the encoder parameters are $\phi = (\eta, \xi)$.

\subsection{Approximate posterior variance}
 
\begin{figure}
    \centering
    \includegraphics[width=\linewidth]{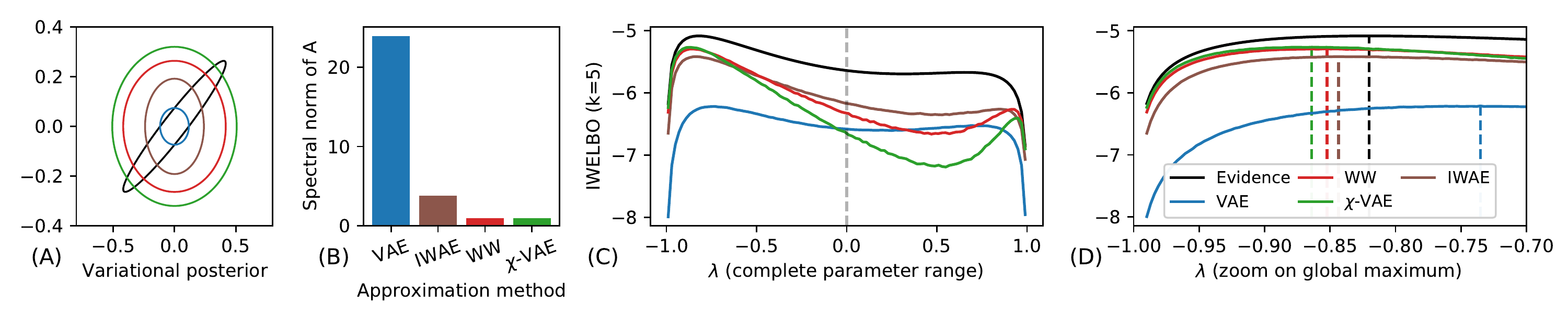}
    \caption{Variational Bayes for the bivariate pPCA example. (A) Gaussian mean-field approximations to the posterior (same legend for all the figures) (B) Corresponding values of $\norm{A}_2$ (C) IWELBO ($k=5$) as a function of $\lambda$, with proposals estimated for $\lambda=0$ and other model parameters fixed to their true values (D) Specific zoom around the global maximum.}
    \label{fig:norm_A}
    \vspace{-0.3cm}
\end{figure}

Exploiting the invariance properties of Gaussian distributions~\cite{wainwright_2019}, our next lemma gives concentration bounds for the logarithm of the importance sampling weights under the posterior.
\begin{restatable}{lemma}{lemmalogratio}\emph{(Concentration of the log-likelihood ratio)}\label{prop:log-ratio}
For an observation $x$, let $\Sigma$ be the variance of the posterior distribution under the pPCA model, $p_\theta(z \mid x)$. Let 
\begin{align}
\label{eq:A}
    A(x) = \Sigma^{\nicefrac{1}{2}}\left[D(x)\right]^{-1}\Sigma^{\nicefrac{1}{2}} - I.
\end{align}
For $z$ following the posterior distribution, $\log w(x, z)$ is a sub-exponential random variable. Further, there exists a $t^*(x)$ such that, under the posterior $p_\theta(z \mid x)$ and for all $t>t^*(x)$,
\begin{align}
    \mathbb{P}\left(\left|\log w(x, z) - \kl{p_\theta}{q_\phi}\right| \geq t\right) & \leq e^{-\frac{t}{8\norm{A(x)}_2}}.
\end{align}
\end{restatable}
This lemma characterizes the concentration of the log-likelihood ratio---a quantity central to all the VAE variants we analyze---as the spectral norm of a simple matrix $\norm{A(x)}_2$. 
Plugging the concentration bound from Lemma~\ref{prop:log-ratio} into the result of~\cite{Chatterjee2018}, we obtain an error bound on the IS estimator for posterior expectations.

\begin{restatable}{thm}{thmppca}\emph{(Sufficient sample size)}\label{thm:linear_VAE}
For an observation $x$, suppose that the second moment of $f(z)$ under the posterior is bounded by $\kappa$. If the number of importance sampling particles $n$ satisfies $n = \beta \exp\{\kl{p_\theta}{q_\phi}\}$ for some $\beta > \log t^*(x)$, then
\begin{align}
\label{eq:tail_sample}
    \mathbb{P}\left(\left|\hat{\mathcal{Q}}^n_{\textrm{IS}}(f, x) -  \mathcal{Q}(f, x) \right| \geq \frac{2\sqrt{3\kappa}}{\beta^{\nicefrac{1}{8\gamma}} - \sqrt{3}}\right) \leq \frac{\sqrt{3}}{\beta^{\nicefrac{1}{8\gamma}}},
\end{align}
with $\gamma = \max\left(1, 4\norm{A(x)}_2\right)$.
\end{restatable}
Theorem~\ref{thm:linear_VAE} identifies a key quantity---the spectral norm of $A(x)$---as useful for controlling the sample efficiency of the IS estimator. The closed-form expression for $\norm{A(x)}_2$ from Eq.~\eqref{eq:A} suggests overestimating the posterior variance is often more suitable than underestimating variance. For the one-dimensional problem,  $\norm{A(x)}_2$ is indeed asymmetric around its global minimum, favoring larger values of $D(x)$.

As a consequence of this result, we can characterize the behavior of several variational inference frameworks such as WW, $\rchi$-VAE, IWAE, and VAE (in this case, the model is fixed). We provide this analysis for a bivariate Gaussian example in which all the quantities of interest can be visualized (for full derivations, see Appendix~\ref{app:biv_gauss}).
Figure~\ref{fig:norm_A}A shows that the VAE underestimates the variance, while other frameworks provide adequate coverage of the posterior. As expected, $\norm{A}_2$ is significantly smaller for WW and the $\rchi$-VAE than for the VAE (Figure~\ref{fig:norm_A}B).

\subsection{Model selection}
In the VAE framework, the model parameters must also be learned. For a fixed variational distribution, variational Bayes (VB) selects the model that maximizes the ELBO or the IWELBO. Each approximate posterior inference method proposes a different lower bound. Even though all these lower bounds become tight (equal to the evidence) with an infinite number of particles, different proposal distribution may not perform equally for model selection with a finite number of particles, as is necessary in practice. Moreover, because the optimal IS proposal depends on the target function~\cite{mcbook}, a good proposal for model learning may not be desirable for decision-making and vice versa. 

We can further refine this statement in the regime with few particles. VB estimates of the model parameters are expected to be biased towards the regions where the variational bound is tighter~\cite{Turner2011}. For a single particle, the tightness of the IWELBO (hence equal to the ELBO) is measured by the reverse KL divergence:
\begin{align}
\label{eq:model_selection}
    \kl{q_\phi}{p_\theta} = \frac{1}{2}\left[\textrm{Tr}\left[\Sigma(\theta)^{-1}D(x)\right] + \log \det{\Sigma(\theta)}\right] + C.
\end{align}
Here $C$ is constant with respect to $\theta$. Because $\kl{p_\theta}{q_\phi}$ is linear in $D(x)$, a higher variance $D(x)$ in Eq.~\eqref{eq:model_selection} induces a higher sensitivity of variational Bayes in parameter space. For multiple particles, no closed-form solution is available so we proceed to numerical experiments on the bivariate pPCA example. We choose five particles. In this setting, the approximate posteriors are fit for an initial value of the parameter $\lambda = 0$ (the real value is $\lambda=0.82$) with all other parameters set to their true value. In Figure~\ref{fig:norm_A}C, WW and the $\rchi$-VAE exhibit a higher sensitivity than the VAE and the IWAE, similar to the case with a single particle. This sensitivity translates into higher bias for selection of $\lambda$ (Figure~\ref{fig:norm_A}D). These results suggest that lower variance proposals may be more suitable for model learning than for decision-making, providing yet another motivation for using different proposals for each task.

\begin{figure}
\begin{minipage}{0.55\textwidth}
\centering
\captionsetup{type=table}
\begin{small}
\begin{sc}
\caption{Results on the pPCA simulated data. MAE refers to the mean absolute error in posterior expectation estimation.}
\label{table:mixture}
\begin{tabular}{lcccc}
\toprule
& \textbf{VAE} & \textbf{IWAE} & \textbf{WW} & \textbf{$\rchi$-VAE}\\
\midrule
$\log p_\theta(X)$  & -17.65 & \textbf{-16.91} & -16.93 & -16.92\\
IWELBO  & -17.66 & \textbf{-16.92} & -16.96 & \textbf{-16.92}\\
\midrule
$\norm{A}_2$  & 1.69 & 1.30 & 2.32 & \textbf{1.13}\\
PSIS  & 0.54 & 0.53 & 0.66 & \textbf{0.47}\\
\midrule
MAE  & 0.103 & 0.032 & 0.043 & \textbf{0.030}\\
\bottomrule
\end{tabular}
\end{sc}
\end{small}
\vspace{-0.3cm}

\end{minipage}
\hfill
\begin{minipage}{0.4\textwidth}
\centering
\includegraphics[width=0.9\textwidth]{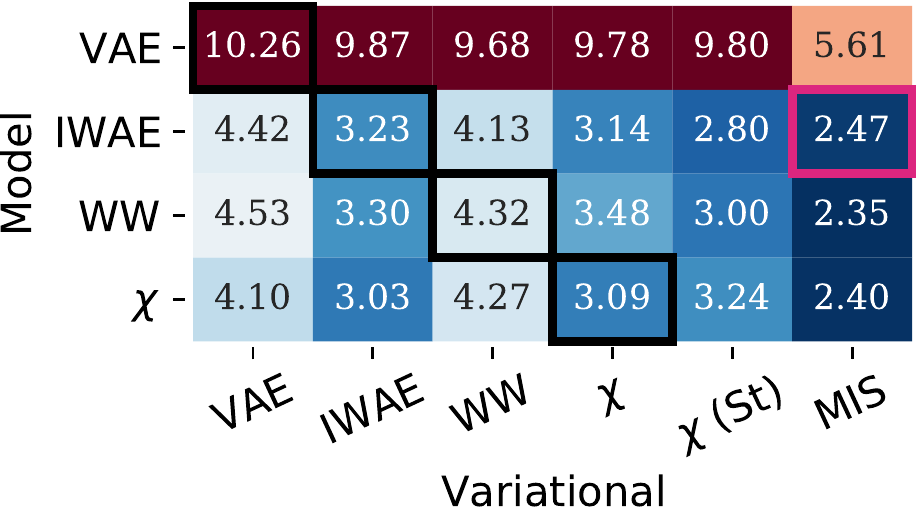}
\caption{MAE ($\times 100$) for the pPCA model. Each row corresponds to an objective function for fitting the model parameters and each column corresponds to an objective function for fitting the variational parameters.
}
\label{fig:pPCA-cross}
\end{minipage}
\vspace{-0.3cm}
\end{figure}

\subsection{pPCA experiments}
\label{posterior_query}
We now investigate the behavior of WW, the $\rchi$-VAE, IWAE, and the VAE for the same model, but in a higher-dimensional setting. In particular, we generate synthetic data according to the pPCA model described in Eq.~\eqref{eq:lin_vae_gen} and assess how well these variants can estimate posterior expectations. Our complete experimental setup is described in Appendix~\ref{app:experiment1}. We use a linear function for the mean and the log-variance of the variational distribution. This is a popular setup for approximate inference diagnostics~\cite{Turner2011} since the posterior has an analytic form, and the form is not typically a mean-field factorization. We propose a toy example of a posterior expectation, $\mathcal{Q}(f_\nu, x) = p_\theta(z_1 \geq \nu \mid x)$, obtained for $f_\nu(z) = \mathds{1}\{e_1^\top z \geq \nu\}$ with $e_1$ as the first vector of the canonical basis of $\mathbb{R}^d$. For this choice of $f_\nu$, the resulting posterior expectation is tractable since it is the cumulative density function of a Gaussian distribution. 

To evaluate the learned generative model, we provide goodness-of-fit metrics based on IWELBO on held-out data, as well as the exact log-likelihood. In addition, we report the Pareto-smoothed importance sampling (PSIS) diagnostic \smash{$\hat{k}$} for assessing the quality of the variational distribution as a proposal~\cite{pmlr-v80-yao18a}. PSIS is computed by fitting a generalized Pareto distribution to the IS weights; it indicates the viability of the variational distribution for IS. For multiple importance sampling (MIS), we combine the proposal (learned on the same model) from IWAE, WW, and $\rchi$-VAE, as well as samples from the prior~\cite{hesterberg1995weighted}, with equal weights.

When not stated otherwise, we report the median PSIS over $64$ observations, using $5,000$ posterior samples.
Unless stated otherwise, we use $30$ particles per iteration for training the models (as in~\cite{rainforth2018tighter}), $10,000$ samples for reporting log-likelihood proxies, and $200$ particles for making decisions. 
All results are averaged across five random initializations of the neural network weights. 

Table~\ref{table:mixture} contains our results for approximating the pPCA model with 5 particles.
IWAE, WW, and the $\rchi$-VAE outperform the VAE in terms of held-out exact likelihood. There is a slight preference for IWAE.
In terms of posterior approximation, all algorithms yielded a reasonable value for the PSIS diagnostic, in agreement with the spectral norm of $A$. PSIS values are not directly comparable between models, as they only measure the suitability of the variational distribution for a particular model.
In terms of mean absolute error (MAE), IWAE, and $\rchi$-VAE achieved the best result. For the VAE, the plugin estimator attained performance similar to the SNIS estimator.

For Figure~\ref{fig:pPCA-cross}, we fix these models and learn several proposal distributions for estimating the posterior expectations. For a fixed model, the proposal from the $\rchi$-VAE often improves the MAE, and the one from MIS performs best. The three-step procedure (IWAE-MIS, shown in red) significantly outperforms all of the single-proposal methods (shown in black). Notably, the performance of WW when used to learn a proposal is not as good as expected. Therefore, we ran the same experiment with a higher number of particles and reported the results in Appendix~\ref{app:experiment1}. 
Briefly, WW learns the best model for 200 particles, and our three-step procedure (WW-MIS) still outperforms all existing methods. 

With respect to the $\rchi$-VAE, using a Student's t distribution as the variational distribution (in place of the standard Gaussian distribution) usually improves the MAE. 
Because the Student's t distribution is also better in terms of theoretical motivation (the heavier tails of the Student's t distribution also help to avoid an infinite $\rchi^2$ divergence), we always use it with the $\rchi$-VAE for the real-world applications.

\section{Classification-based decision theory}
\label{label_decision}

We consider the MNIST dataset, which includes features $x$ for each of the images of handwritten digits and a label $c$. We aim to predict whether $x$ has its label included in a given subset of digits (from zero to eight). We also allow no decision to be made for ambiguous cases (the ``reject'' option)~\cite{Bishop:2006:PRM:1162264}.

We split the MNIST dataset evenly between training and test datasets.
For the labels 0 to 8, we use a total of $1,500$ labeled examples. All images associated with label 9 are unlabelled. We assume in this experiment that the MNIST dataset has $C = 9$ classes. For $c \in \{0, \ldots, C-1\}$, let $p_\theta(c \mid x)$ denote the posterior probability of the class $c$ for a model yet to be defined. Let $L(a, c)$ be the loss defined over the action set $\mathcal{A} = \{\varnothing, 0, \ldots, C-1\}$. Action $\varnothing$ is known as the rejection option in classification (we wish to reject label 9 at decision time). For this loss, it is known that the optimal decision $a^*(x)$ is a threshold-based rule~\cite{Bishop:2006:PRM:1162264} on the posterior probability. This setting is fundamentally different than traditional classification because making an informed decision requires knowledge of the full posterior $p_\theta(c \mid x)$, not just the maximum probability class. The Bayes optimal decision rule is based on the posterior expectation $\mathcal{Q}(f, x)$ for $z = c$, where $f$ a constant unit function. We provide a complete description of the experimental setup, as well as derivations for the plugin estimator and the SNIS estimator, in Appendix~\ref{app:experiment2}. To our knowledge, this is the first time semi-supervised generative models have been evaluated in a rejection-based decision-making scenario. 

As a generative model, we use the M1+M2 model for semi-supervised learning~\cite{KingmaRMW14}.
In the M1+M2 model, the discrete latent variable $c$ represents the class. Latent variable $u$ is a low-dimensional vector encoding additional variation. Latent variable $z$ is a low-dimensional representation of the observation. It is drawn from a mixture distribution with mixture assignment $c$ and mixture parameters that are a function of $u$.
The generative model is
\begin{align}
\label{eq:m1m2_gen}
    p_\theta(x, z, c, u) &= p_\theta(x \mid z)p_\theta(z \mid c, u)p_\theta(c)p_\theta(u).
\end{align}
The variational distribution factorizes as
\begin{align}
q_\phi(z, c, u \mid x) &= q_\phi(z \mid x)q_\phi(c \mid z)q_\phi(u \mid z, c).
\end{align}
Because the reverse KL divergence can cover only one mode of the distribution, it is prone to attributing zero probability to many classes; some other divergences would penalize this behavior. Appendix~\ref{app:m1m2} further describes why the M1+M2 model trained as a VAE may be overconfident and why WW and the $\rchi$-VAE can remedy this problem. To simplify the derivation of updates for the M1+M2 model with all of the algorithms we consider, we use the Gumbel-softmax trick~\cite{jang2017categorical} for latent variable $c$ for unlabelled observations. We fit the M1+M2 model with only nine classes and consider as ground truth that images with label 9 should be rejected at decision time.

\begin{figure}
\begin{minipage}{0.55\textwidth}
\centering
\captionsetup{type=table}
\begin{small}
\begin{sc}
\caption
{
Results for the M1+M2 model on MNIST. AUPRC refers to the area under of the PR curve for rejecting the label 9. 
}
\label{table:m1m2}
\begin{tabular}{lccccc}
\toprule
 & \textbf{VAE} & \textbf{IWAE}& \textbf{WW} & \textbf{$\rchi$-VAE} \\ 
 \midrule
IWELBO  & -104.74 & \textbf{-101.92} & -102.82 & -105.29\\
PSIS & $\gtrsim$ 1 & $\gtrsim$ 1 & $\gtrsim$ 1 & $\gtrsim$ 1\\
AUPRC & 0.35 & \textbf{0.45} & 0.29 & 0.44 \\
\bottomrule
\end{tabular}
\end{sc}
\end{small}
\vspace{-0.3cm}
\end{minipage}
\hfill
\begin{minipage}{0.4\textwidth}
\centering
\includegraphics[width=0.8\textwidth]{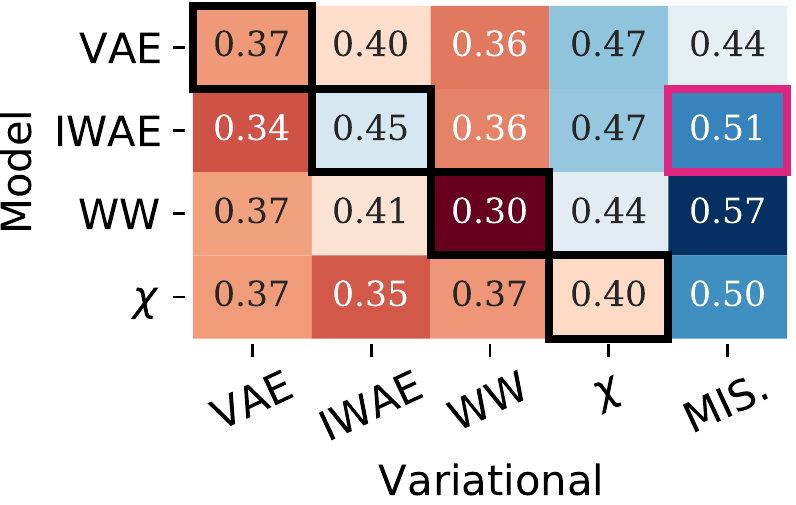}
\caption{AUPRC for MNIST.
Each row corresponds to an objective function for fitting the model parameters and each column corresponds to an objective function for fitting the variational parameters.
}
\label{fig:MNIST-cross}
\end{minipage}
\vspace{-0.4cm}
\end{figure}

Table~\ref{table:m1m2} gives our results, including the area under the precision-recall curve (AUPRC) for classifying ``nines'', as well as goodness-of-fit metrics. IWAE and WW learn the best generative model in terms of IWELBO. For all methods, we compare the classification performance of the plugin and the SNIS estimator on labels 1 through 8. While the plugin estimator shows high accuracy across all methods (between $95\%$ and $97\%$), the SNIS estimator shows poor performance (around $60\%$). This may be because the PSIS diagnosis estimates are greater than one for all algorithms, which indicates that the variational distribution may lead to large estimation error if used as a proposal~\cite{pmlr-v80-yao18a}. 
Consequently, for this experiment we report the result of the plugin estimator. Figure~\ref{fig:MNIST-cross} shows the results of using different variational distributions with the model held fixed. These results suggest that using the $\rchi$-VAE or MIS on a fixed model leads to better decisions. Overall, our three-step procedure (IWAE-MIS) outperforms all single-proposal alternatives.

\section{Multiple testing and differential gene expression}
\label{FDR}

We present an experiment involving Bayesian hypothesis testing for detecting differentially expressed genes from single-cell transcriptomics data, a central problem in modern molecular biology. Single-cell RNA sequencing data (scRNA-seq) provides a noisy snapshot of the gene expression levels of cells~\cite{Wagner2016, Tanay2017}. It can reveal cell types and gene markers for each cell type as long as we have an accurate noise model~\cite{Grun2014}. scRNA-seq data is a cell-by-gene matrix $X$.
For cell $n=1,\ldots,N$ and gene $g=1,\ldots,G$,
entry $X_{ng}$ is the number of transcripts for gene $g$ observed in cell $n$.
Here we take as given that each cell comes paired with an observed cell-type label $c_n$. 

Single-cell Variational Inference (scVI)~\cite{Lopez292037} is a hierarchical Bayes model~\cite{GelmanHill:2007} for scRNA-seq data. For our purposes here, it suffices to know that the latent variables $h_{ng}$ represent the underlying gene expression level for gene $g$ in cell $n$, corrected for a certain number of technical variations (e.g., sequencing depth effects). The corrected expression levels $h_{ng}$ are more reflective of the real proportion of gene expression than the raw data $X$~\cite{Cole2017}. Log-fold changes based on $h_{ng}$ can be used to detect differential expression (DE) of gene $g$ across cell types $a$ and $b$~\cite{deseq2,Boyeau794289}. Indeed, Bayesian decision theory helps decide between a model $\mathcal{M}_1^g$ in which gene $g$ is DE and an alternative model $\mathcal{M}_0^g$ in which gene $g$ is not DE. The hypotheses are
\begin{align}
    \mathcal{M}^g_1: \left|\log\frac{h_{ag}}{h_{bg}}\right| \geq \delta ~~~~ \text{and}~~~ \mathcal{M}^g_0: \left|\log\frac{h_{ag}}{h_{bg}}\right| < \delta,
    \label{hypotheses}
\end{align}
where $\delta$ is a threshold defined by the practitioner. DE detection can therefore be performed by posterior estimation of log-fold change between two cells, $x_a$ and $x_b$, by estimating
\(
p_\theta(\mathcal{M}^g_1 \mid x_a, x_b)
\)
with importance sampling.
The optimal decision rule for 0-1 loss is attained by thresholding the posterior log-fold change estimate.
Rather than directly setting this threshold, a practitioner typically picks a false discovery rate (FDR) $f_0$ to target. To control the FDR, 
we consider the multiple binary decision rule $\mu^k = (\mu_g^k, g \in G)$ that consists of tagging as DE the $k$ genes with the highest estimates of log-fold change.
With this notation, the false discovery proportion (FDP) of such a decision rule is
\begin{align}
    \mathrm{FDP} = 
        \frac
        {\sum_g (1 - \mathcal{M}_1^g) \mu_g^k}
        {\sum_g \mu_g^k}.
\end{align}
Following~\cite{Cui2015}, we define the posterior expected FDR as
\(
\overline{\mathrm{FDR}} := \mathbb{E}\left[\mathrm{FDP} \mid x_a, x_b\right],
\)
which can be computed from the differential expression probabilities of Eq.~\eqref{hypotheses}. We then set $k$ to the maximum value for which the posterior expected FDR is below $f_0$. 
In this case, controlling the FDR requires estimating $G$ posterior expectations. 
To quantify the ability of each method to control the FDR, we report the mean absolute error between the ground-truth FDR and the posterior expected FDR.

\begin{figure}
\begin{minipage}{0.55\textwidth}
\centering
\captionsetup{type=table}
\caption
{
Results for the scVI model. MAE FDR refers to the mean absolute error for FDR estimation.}
\label{table_fdr}
\begin{center}
\begin{small}
\begin{sc}
\begin{tabular}{lccccccc}
\toprule
 & \textbf{VAE}  & \textbf{IWAE} & \textbf{WW} & \textbf{$\chi$-VAE} \\ \midrule
AIS  & -380.42 & -372.86 & -372.78 & \textbf{-371.69}\\
IWELBO  & -380.42 & -372.86 & -372.78 & \textbf{-371.69}\\
\midrule
PSIS  & 0.71 & 0.49 & \textbf{0.47} & 0.69\\
\midrule
MAE FDR  & 5.78 & 0.39 & 0.51 & \textbf{0.27}\\
\bottomrule
\end{tabular}
\end{sc}
\end{small}
\end{center}
\vspace{-0.3cm}
\end{minipage}
\hfill
\begin{minipage}{0.4\textwidth}
\centering
\includegraphics[width=0.8\textwidth]{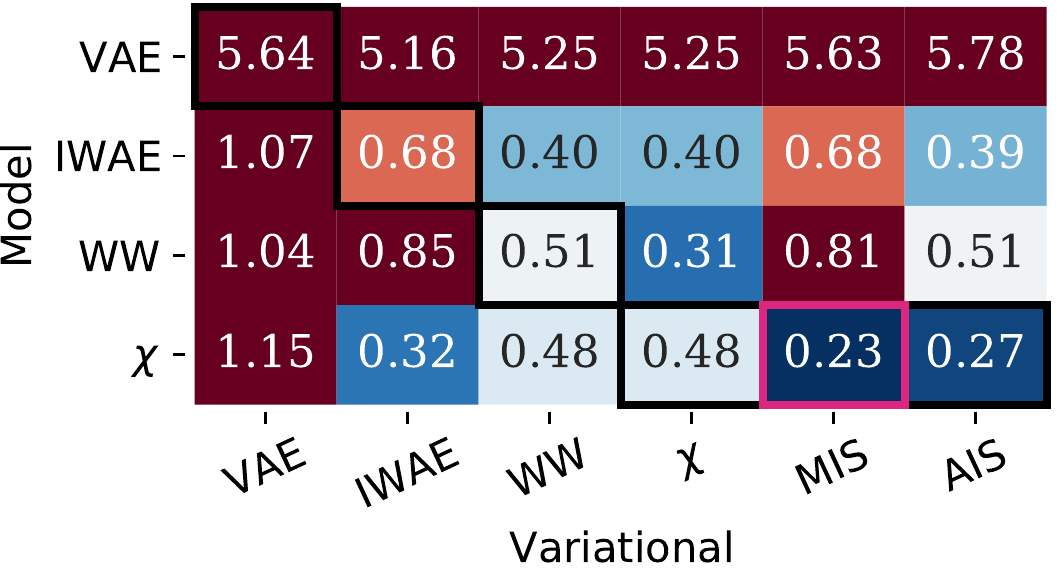}
\caption{FDR MAE for scVI. 
Each row corresponds to an objective function for fitting the model parameters and each column corresponds to an objective function for fitting the variational parameters.
}\label{fig:scVI-cross}
\end{minipage}
\vspace{-0.3cm}
\end{figure}

We fitted the scVI model with all of the objective functions of interest. In addition, we used annealed importance sampling (AIS)~\cite{neal2001annealed} to approximate the posterior distribution once the model was fitted. AIS is computationally intensive, so we used $500$ steps and $100$ samples from the prior to keep the runtime manageable. Because the ground-truth FDR cannot be computed on real data, we generated scRNA-seq data according to a Poisson log-normal model (a popular choice for simulating scRNA-seq data~\cite{townes2019feature,Crowell713412}) from five distinct cells types corresponding to $10,000$ cells and $100$ genes. This model is meant to reflect both the biological signal and the technical effects of the experimental protocol. We provide further details of this experiment in Appendix~\ref{app:experiment3}.

VAE performs worst in terms of held-out log-likelihood, while $\rchi$-VAE performs best (Table~\ref{table_fdr}).
We also evaluated the quality of the gene ranking for differential expression based on the AUPRC of the DE probabilities. The model learned with the VAE objective function has an AUPRC of 0.85, while all other combinations have an AUPRC of more than 0.95. 

Next, we investigate the role posterior uncertainty plays in properly estimating the FDR.  
Table \ref{table_fdr} and Figure~\ref{fig:scVI-cross} report the mean absolute error (MAE) for FDR estimation over the 100 genes. The posterior expected FDR has a large error for any proposal based on the VAE model, which gets even worse using the plugin estimator. However, all of the other models yield a significantly lower error for any proposal. We provided the FDR curves in Appendix~\ref{app:experiment3}. The VAE has highly inaccurate FDR control. 
This suggests that the model learned by the original scVI training procedure (i.e., a VAE) cannot be used for FDR control even with AIS as the proposal. 
IWAE, WW, and $\rchi$-VAE, on the other hand, may be useful to a practitioner. Further, in this experiment WW slightly outperforms IWAE in terms of held-out IWELBO but not in terms of FDR MAE. Again, the variational distribution used to train the best generative model does not necessarily provide the best decisions.
Conversely, our three-step procedure ($\rchi$-VAE-MIS) has the best FDR estimates for this experiment, improving over all the single-proposal alternatives, as well as over the $\rchi$-VAE in conjunction with AIS. 
\section{Discussion}
We have proposed a three-step procedure for using variational autoencoders for decision-making. This method is theoretically motivated by analyzing the derived error of the self-normalized importance sampling estimator and the biases of variational Bayes on the pPCA model. Our numerical experiments show that in important real-world examples this three-step procedure outperforms VAE, IWAE, WW, $\rchi$-VAE, and IWAE in conjunction with AIS.

The proposed procedure requires fitting three VAEs, each with a different loss function. This causes computational overhead in our method compared to a standard VAE. However, this overhead is not large (roughly a constant factor of three) since our method simply consists of training three VAEs, each with a different loss function. In the pPCA experiment, training a single VAE takes 12 seconds, while step one and two of our method together take 53 seconds (on a machine with a single NVIDIA GeForce RTX 2070 GPU). Step three of our method does not introduce any additional computations. In cases where an offline decision is made (for example, in biology), we do not expect the overhead of our method to create a bottleneck.

Posterior collapse is a much studied failure mode of VAEs in which the variational distribution equals the prior and the conditional likelihood of the data is independent of the latent variables~\cite{NIPS2019_9138}.
In Appendix~\ref{app:experiment1}, we show that two techniques proposed to mitigate the problem of posterior collapse, cyclical annealing~\cite{fu2019cyclical} and lagging inference networks~\cite{he2018lagging}, do indeed improve the performance of VAEs. However, the three-step procedure we propose still outperforms them.

A complementary approach to our own for decision-making is the elegant framework of loss-calibrated inference~\cite{lacoste2011approximate,NIPS2019_8868}, which adapts the ELBO to take into account the loss function $L$. Similarly, amortized Monte Carlo integration~\cite{golinski2019amortized} proposes to fit a variational distribution for a particular expectation of the posterior. Neither approach is directly applicable to our setting because adapting the ELBO to a specific decision-making loss implies a bias in learning $p_\theta$. However, these approaches could potentially improve steps two and three of our proposed framework. Consequently, developing hybrid algorithms is a promising direction for future research.

\section*{Broader impact}

The method we propose allows practitioners to make better decisions based on data. This capability may improve decisions about topics as diverse as differential expression in biological data, supply chain inventory and pricing, and personalized medicine treatments. However, the black-box nature of neural networks---a key aspect of our approach---confers both benefits and risks. For data without a simple parametric distribution, neural networks allow us to fit a model accurately, so that we can make decisions in a rigorous data-driven way without recreating prior biases. However, it can be difficult to check the quality of the model fit, particularly when worst-case analysis is appropriate, e.g., in mission-critical applications. In VAE-style architectures, powerful decoder networks are associated with posterior collapse, which could go undetected. More research is needed to ensure the worst-case performance of VAE-style models and/or to diagnosis fit problems before they are used in high-stakes decision-making scenarios.

\begin{ack}
We thank Chenling Xu for suggestions that helped with the third experiment. We thank Adam Kosiorek and Tuan-Anh Le for answering questions about implementation of the Wake-Wake algorithm. We thank Geoffrey Negiar, Adam Gayoso, and Achille Nazaret for their feedback on this article. NY and RL were supported by grant U19 AI090023 from NIH-NIAID. 
\end{ack}

\bibliographystyle{unsrt}
{\footnotesize
\bibliography{biblio}}

\newpage

\appendix
\part*{Appendices}
 
In Appendix~\ref{app:proofs}, we prove the theoretical results presented in this manuscript. Appendix~\ref{app:chi-vaes} presents details of the $\rchi$-VAE. Appendix~\ref{app:err_simple_bounds} provides known bounds for posterior expectation estimators. In Appendix~\ref{app:biv_gauss}, we present the analytical derivations in the bivariate Gaussian setting. 
Appendix~\ref{app:experiment1} and~\ref{app:experiment2} provide details of the pPCA and MNIST experiments. In Appendix~\ref{app:m1m2}, we discuss why alternative divergences would be especially suitable for the M1+M2 model. Appendix~\ref{app:experiment3} presents additional information about the single-cell transcriptomics experiment.

\section{Proofs}
\label{app:proofs}

\subsection{Proof of Lemma~\ref{prop:log-ratio}}

\lemmalogratio*

\begin{proof}
Here we first give the closed-form expression of the posterior and then prove the concentration bounds on the log-likelihood ratio. Let $M = W^\top W+\sigma^2I$. For notational convenience, we do not explicitly denote dependence on random variable $x$.

\textit{Step 1: Tractable posterior.} Using the Gaussian conditioning formula~\cite{Bishop:2006:PRM:1162264}, we have that 
\begin{align}
    p_\theta(z \mid x) = \textrm{Normal}\left(M^{-1}W^\top (x - \mu), \sigma^2M^{-1}\right).
\end{align} 

\textit{Step 2: Concentration of the log-ratio.} For this, since $x$ is a fixed point, we note $a = M^{-1}W^\top (x - \mu)$ and $b = \nu(x)$. We can express the log density ratio as
\begin{align}
    w(z, x) &= \log \frac{p_\theta(z \mid x)}{q_\phi(z \mid x)} \\
    &= -\frac{1}{2}\log\det(\sigma^2M^{-1} D^{-1}) -\frac{1}{2\sigma^2}(z - a)^\top M(z-a) + \frac{1}{2}(z - b)^\top D^{-1}(z - b) \\
    &= C + z^\top [\frac{D^{-1}}{2}-\frac{M}{2\sigma^2}]z + [D^{-1}b -\frac{Ma}{\sigma^2}]^\top z,
\end{align}
where $C$ is a constant. To further characterize the tail behavior, let $\epsilon$ be an isotropic multivariate normal distribution, and let us express the log-ratio as a function of $\epsilon$ instead of the posterior probability. We have that $z = M^{-1}W^\top (x - \mu) + \sigma M^{-\nicefrac{1}{2}}\epsilon$. The log ratio can now be written as
\begin{align}
    \log w(z, x) &= C' + \epsilon^\top [\frac{\sigma^2M^{-\nicefrac{1}{2}}D^{-1}M^{-\nicefrac{1}{2}} - I}{2}]\epsilon + [\sigma M^{-\nicefrac{1}{2}}D^{-1}b -\frac{M^{\nicefrac{1}{2}}a}{\sigma}]^\top \epsilon.
\end{align}
Because $\epsilon$ is isotropic Gaussian, we can compute the deviation of this log-ratio and provide concentration bounds. Because $\epsilon$ is Gaussian and $\epsilon \mapsto \log w(z, x)$ is a quadratic function, we show that the log-ratio under the posterior is a sub-exponential random variable. 

The following lemma makes this statement precise, and it carries an implication similar to the classic result in~\cite{laurent2000}.
\begin{lemma}
\label{lemma:sub_exp}
Let $d \in \mathbb{N}^*$ and $\epsilon \sim \mathrm{Normal}(0, I_d)$. For matrix $A \in \mathbb{R}^{d \times d}$ and vector $b \in \mathbb{R}^d$, random variable $v = \epsilon^\top A\epsilon + b^\top \epsilon$ is sub-exponential with parameters $(\sqrt{2\norm{A}^2_F + \nicefrac{\norm{b}_2^2}{4}}, 4\norm{A}_2)$. In particular, we have the following concentration bounds:
\begin{align}
    \mathbb{P}\left[|v| \geq t\right] & \leq 2 \exp\left\{-\frac{t^2}{8\norm{A}^2_F + \norm{b}_2^2 + 4\norm{A}_2t}\right\}~~~~~\textrm{for all}~t >0. \\
    \mathbb{P}\left[|v| \geq t\right] & \leq \exp\left\{-\frac{t}{8\norm{A}_2}\right\}~~~~~\textrm{for all}~t > \frac{8\norm{A}^2_F + \norm{b}_2^2}{16\norm{A}_2}. 
\end{align}
\end{lemma}
For $A = \frac{\sigma^2M^{-\nicefrac{1}{2}}D^{-1}M^{-\nicefrac{1}{2}} - I}{2}$ and $b = \sigma M^{-\nicefrac{1}{2}}D^{-1}b -\frac{M^{\nicefrac{1}{2}}a}{\sigma}$, we can apply Lemma~\ref{lemma:sub_exp}. We deduce a concentration bound on the log-ratio around its mean, which is the forward Kullback-Leibler divergence $L = \kl{p_\theta(z \mid x)}{q_\phi(z \mid x)}$. More precisely, we have that
\begin{align}
    p_\theta\left(\left|\log w(z, x) - L\right| \geq t \mid x\right) & \leq 2 \exp\left\{-\frac{t^2}{8\norm{A}^2_F + \norm{b}_2^2 + 4\norm{A}_2t}\right\}~~~~~\textrm{for all}~t >0,
\end{align}
\end{proof}
as well as the deviation bound for large $t$, which ends the proof. 

\subsection{Proof of Lemma~\ref{lemma:sub_exp}}

\begin{proof}
Let $\lambda \in \mathbb{R}^+$. We have that $\mathbb{E}v = \Tr(A)$. We wish to bound the moment generating function
\begin{align}
    \mathbb{E}[e^{\lambda(v-\Tr(A))}] = e^{-\lambda \Tr(A)} \mathbb{E}[e^{\lambda(\epsilon^\top A\epsilon + b^\top \epsilon)}].
\end{align}
Sums of arbitrary correlated variables are hard to analyze.
Here we rely on the property that Gaussian vectors are invariant under rotation. Let $A = Q\Lambda Q^\top $ be the eigenvalue decomposition for $A$ and denote $\epsilon = Q\xi$ and $b = Q\beta$. Since $Q$ is an orthogonal matrix, $\xi$ also follows an isotropic normal distribution and
\begin{align}
    \mathbb{E}[e^{\lambda(v-\Tr(A))}] &= e^{-\lambda \Tr(A)} \mathbb{E}[e^{\lambda(\xi^\top \Lambda\xi + \beta^\top \xi)}] \\
    &= e^{-\lambda \Tr(A)} \mathbb{E}\left[ \prod_{i=1}^d e^{\lambda\xi_i^2\Lambda_i + \lambda \beta_i\xi_i}\right] \\
    &= e^{-\lambda \Tr(A)} \prod_{i=1}^d \mathbb{E} \left[e^{\lambda\xi_i^2\Lambda_i + \lambda \beta_i\xi_i}\right] \\
    &= \prod_{i=1}^d \mathbb{E} \left[e^{\lambda\xi_i^2\Lambda_i + \lambda \beta_i\xi_i - \lambda\Lambda_i}\right].
\end{align}
Because each component $\xi_i$ follows a isotropic Gaussian distribution, we can compute the moment generating functions in closed form:
\begin{align}
    \mathbb{E} \left[e^{\lambda\xi_i^2\Lambda_i + \lambda \beta_i\xi_i - \lambda\Lambda_i}\right] &= 
    \frac{e^{-\lambda\Lambda_i}}{\sqrt{2\pi}}\int_{-\infty}^{+\infty} e^{\lambda\Lambda_iu^2 + \lambda\beta_i u}e^{-\frac{u^2}{2}}du \\
    &=
    \frac{e^{-\lambda\Lambda_i}}{\sqrt{2\pi}}\int_{-\infty}^{+\infty} e^{[\lambda\Lambda_i -\frac{1}{2}]u^2 + \lambda\beta_i u}du.
\end{align}
This integral is convergent if and only if $\lambda < \nicefrac{1}{2\Lambda_i}$. In that case, after a change of variable, we have that
\begin{align}
    \mathbb{E} \left[e^{\lambda\xi_i^2\Lambda_i + \lambda \beta_i\xi_i - \lambda\Lambda_i}\right] &= 
    \frac{e^{-\lambda\Lambda_i}}{\sqrt{\pi}\sqrt{1 - 2\lambda\Lambda_i}}\int_{-\infty}^{+\infty} e^{-s^2 + \frac{\sqrt{2}\lambda\beta_i s}{\sqrt{1 - 2\lambda\Lambda_i}}}ds \\
    &= 
    \frac{e^{-\lambda\Lambda_i}e^{\frac{\lambda^2\beta_i^2}{2(1-2\lambda\Lambda_i)}}}{\sqrt{1 - 2\lambda\Lambda_i}}.
\end{align}
Then, using the fact that for $a < \nicefrac{1}{2}$, we have $e^{-a} \leq e^{2a^2}\sqrt{1 - 2 a}$, we can further simplify for $\lambda < \frac{1}{4\Lambda_i}$
\begin{align}
    \mathbb{E} \left[e^{\lambda\xi_i^2\Lambda_i + \lambda \beta_i\xi_i - \lambda\Lambda_i}\right] &\leq
    e^{2\lambda^2\Lambda_i^2 + \frac{\lambda^2\beta_i^2}{2(1-2\lambda\Lambda_i)}} \\
    &\leq
    e^{[2\Lambda_i^2 + \frac{\beta_i^2}{4}]\lambda^2}.
\end{align}
Putting back all the components of $\xi$, we have that for all $\lambda < \frac{1}{4\norm{\Lambda}_2} = \frac{1}{4\norm{A}_2}$
\begin{align}
    \mathbb{E}[e^{\lambda(v-\Tr(A))}] &\leq \exp\left\{\left[2\norm{\Lambda}^2_F + \frac{\norm{\beta}_2^2}{4}\right]\lambda^2\right\} \\
    &\leq \exp\left\{\left[2\norm{A}^2_F + \frac{\norm{b}_2^2}{4}\right]\lambda^2\right\},
\end{align}
where the last inequality is in fact an equality because $Q$ is an isometry. Therefore, according to Definition 2.2 in~\cite{wainwright_2019}, $v$ is sub-exponential with parameters $(\sqrt{2\norm{A}^2_F + \nicefrac{\norm{b}_2^2}{4}}, 4\norm{A}_2)$. The concentration bound is derived as in the proof of Proposition 2.3 in~\cite{wainwright_2019}.
\end{proof}

\subsection{Proof of Theorem~\ref{thm:linear_VAE}}

\thmppca*

\begin{proof}
Let $t = \ln \beta$. By Theorem 1.2 from~\cite{Chatterjee2018}, Lemma 1 for $t > t^*(x)$, and 
\begin{align}
\epsilon = \left(e^{-\frac{t}{4}} + 2e^{-\frac{t}{16\norm{A(x)}_2}}\right)^{\nicefrac{1}{2}},
\end{align}
we have that
\begin{align}
    \mathbb{P}\left(\left|\hat{\mathcal{Q}}^n_{\textrm{IS}}(f, x) -  \mathcal{Q}(f, x) \right| \geq \frac{2 \norm{f}_2\epsilon}{1 - \epsilon}\right) \leq \epsilon.
\end{align}
Now, let us notice that
\(
    \epsilon \leq \sqrt{3}e^{\frac{-t}{8\gamma}}
\)
and that $x \mapsto \nicefrac{x}{1-x}$ is increasing on $(0, 1)$. So we have that 
\begin{align}
    \mathbb{P}\left(\left|\hat{\mathcal{Q}}^n_{\textrm{IS}}(f, x) -  \mathcal{Q}(f, x) \right| \geq \frac{2\sqrt{3\kappa}}{e^{\frac{t}{8\gamma}} - \sqrt{3}}\right) \leq \sqrt{3}e^{\frac{-t}{8\gamma}}.
\end{align}
The bound in Theorem~\ref{thm:linear_VAE} follows by replacing $e^t$ by $\beta$ in the previous equation.
\end{proof}

\section{Chi-VAEs}
\label{app:chi-vaes}
We propose a novel variant of the WW algorithm based on $\chi^2$ divergence minimization, which is potentially well suited for decision-making. This variant is incremental in the sense that it combines several existing contributions such as the CHIVI procedure~\cite{NIPS2017_6866}, the WW algorithm~\cite{le2018revisiting}, and use of a reparameterized Student's t distributed variational posterior (e.g., the one explored in~\cite{NIPS2018_7699} for IWVI). However, we did not encounter prior mention of such a variant in the existing literature. 

In the $\rchi$-VAE, we update the model and the variational parameters as a first-order stochastic block coordinate descent (as in WW~\cite{le2018revisiting}). For a fixed inference model $q_\phi$, we take gradients of the IWELBO with respect to the model parameters. For a fixed generative model $p_\theta$, we seek to minimize the $\rchi^2$ divergence between the posterior and the inference model. This quantity is intractable, but we can formulate an equivalent optimization problem using the $\rchi$ upper bound (CUBO)~\cite{NIPS2017_6866}:
\begin{align}
\begin{split}
    \underbrace{\log p_\theta(x)}_{\text{evidence}} &= \underbrace{\frac{1}{2}\log\mathbb{E}_{q_\phi(z \mid x)}\left(\frac{p_\theta(x, z)}{q_\phi(z \mid x)}\right)^2}_{\text{CUBO}} - \underbrace{\frac{1}{2}\log\left(1 + \chidiv{p_\theta}{q_\phi}\right)}_{\text{$\chi^2$ VG}}. \label{CUBO}
\end{split}
\end{align}
It is known that the properties of the variational distribution (mode-seeking or mass-covering) highly depend on the geometry of the variational gap~\cite{futami2018variational}, which was our initial motivation for using the $\rchi$-VAE for decision-making. For a fixed model, minimizing the exponentiated CUBO is a valid approach for minimizing the $\rchi^2$ divergence. 

Finally, in many cases the $\rchi^2$ divergence may be infinite. This is true even for two Gaussian distributions provided that the variance of $q_\phi$ does not cover the posterior sufficiently. In our pPCA experiments, we found that using a Gaussian distributed posterior may still provide helpful proposals. However, we expect a Student's t distributed variational posterior to properly alleviate this concern. \cite{NIPS2018_7699} proposed a reparameterization trick for elliptical distributions, including the Student's t distribution based on the CDF of the $\rchi$ distribution. In our experiments, we reparameterized samples from a Student's t distribution with location $\mu$, scale $\Sigma = A^\top A$ and degrees of freedom $\nu$ as follows:
\begin{align}
    \delta &\sim \textrm{Normal}(0, I) \\
    \epsilon &\sim \chi^2_\nu\\
    t &\sim \sqrt{\frac{\nu}{\epsilon}}A^\top \delta + \mu,
\end{align} 
where we used reparameterized samples for the $\rchi^2_\nu$ distribution following~\cite{figurnov2018implicit}.

\section{Limitations of standard results for posterior statistics estimators}
\label{app:err_simple_bounds}

Neither~\cite{Chatterjee2018} nor~\cite{Agapiou2017} suggest that upper bounds on the error of the IS estimator may be helpful in comparing algorithmic procedures. Similarly, \cite{le2018revisiting} used the result from~\cite{Chatterjee2018} as a motivation but did not use it to support a claim of better performance over other methods. In this section, we outline two simple reasons why upper bounds on the error of IS are not helpful for comparing algorithms. 

We start by stating simple results of upper bounding the mean square error of the SNIS estimator. 
\begin{restatable}{prop}{propbounds}\label{prop:simple_bounds}\emph{(Deviation for posterior expectation estimates)} Let $f$ be a bounded test function. For the plugin estimator, we have
\begin{align}
\label{eq:up_plugin}
    \sup_{\norm{f}_\infty \leq 1}\mathbb{E}\left[ \left( \hat{\mathcal{Q}}^n_{\textrm{P}}(f, x) -  \mathcal{Q}(f, x) \right)^2 \right]     \leq & 4\tvsq{p_\theta}{q_\phi} + \frac{1}{2n},
\end{align}
where $\Delta_\textrm{TV}$ denotes the total variation distance. For the SNIS estimator, if we further assume that $w(x, z)$ has a finite second-order moment under $q_\phi(z \mid x)$, then we have
\begin{align}
\label{eq:up_snis}
    \sup_{\norm{f}_\infty \leq 1}\mathbb{E}\left[ \left( \hat{\mathcal{Q}}^n_{\textrm{IS}}(f, x) -  \mathcal{Q}(f, x) \right)^2 \right] \leq & \frac{4\chidiv{p_\theta}{q_\phi}}{n},
\end{align}
where $\Delta_{\chi^2}$ denotes the chi-square divergence.
\end{restatable}
We derive the first bound in a later section; the second is from~\cite{Agapiou2017}. We now derive two points to argue that such bounds are uninformative for selecting the best algorithm. 

First, these bounds suggest the plugin estimator is suboptimal because, in contrast to the SNIS estimator, its bias does not vanish with infinite samples. 
However, the upper bound in~\cite{Agapiou2017} may be uninformative when the $\chi^2$ divergence is infinite (as it may be for a VAE). Consequently, it is not immediately apparent which estimator will perform better. 

A second issue we wish to underline pertains to the general fact that upper bounds may be loose. For example, with Pinsker's inequality we may further upper bound the bias of the plugin estimator by the square root of either $\kl{p_\theta}{q_\phi}$ or $\kl{q_\phi}{p_\theta}$. In this case, the VAE and the WW algorithm~\cite{le2018revisiting} both minimize an upper bound on the mean-square error of the plugin estimator; the one we should choose is again unclear.

\subsection{Proof of Proposition 1}
\propbounds*
\begin{proof}
For the plugin estimator
\begin{align}
\hat{\mathcal{Q}}^n_{\textrm{P}}(f, x) = \frac{1}{n}\sum_{i=1}^nf(z_i),
\end{align}
we can directly calculate and upper bound the mean-square error. First, for notational convenience, we will use
\begin{align}
    I^* &= \mathbb{E}_{p(z \mid x)} f(z) \\
    \bar{I} &= \mathbb{E}_{q(z \mid x)} f(z).
\end{align}
Observe that
\begin{align}
    \left|I^* - \bar{I}\right| \leq & \sup_{\norm{g}_\infty \leq 1}\left|\mathbb{E}_{q(z \mid x)} g(z) - \mathbb{E}_{p(z \mid x)} g(z) \right| \\
    &\leq 2\tv{p(z \mid x)}{q(z \mid x)},
\end{align}
by definition of the total variation distance. Now we can proceed to the calculations
\begin{align}
    \left( \frac{1}{n}\sum_{i=1}^nf(z_i) - I^*  \right)^2  & = \left( \frac{1}{n}\sum_{i=1}^nf(z_i) - \bar{I}\right)^2 + (I^* - \bar{I})^2 + 2\left( \frac{1}{n}\sum_{i=1}^nf(z_i) - \bar{I}\right)(I^* - \bar{I}),
\end{align}
and take expectations on both sides with respect to the variational distribution:
\begin{align}
\mathbb{E}\left( \frac{1}{n}\sum_{i=1}^nf(z_i) - I^*  \right)^2  & = \mathbb{E}\left( \frac{1}{n}\sum_{i=1}^nf(z_i) - \bar{I}\right)^2 + (I^* - \bar{I})^2 \\
& = \frac{1}{n}\mathbb{E}\left(f(z_1) - \bar{I}\right)^2 + (I^* - \bar{I})^2 \\
& \leq \frac{1}{2n} + 4\tvsq{p(z \mid x)}{q(z \mid x)}.
\end{align}

For the self-normalized importance sampling estimator
\begin{align}
\hat{\mathcal{Q}}^n_{\textrm{IS}}(f, x) = 
    \frac{1}{n}\sum_{i=1}^nw(x, z_i)f(z_i),
\end{align}
we instead rely on Theorem 2.1 of~\cite{Agapiou2017}.
\end{proof}

\section{Analytical derivations in the bivariate Gaussian setting}
\label{app:biv_gauss}
For a fixed $x$, we adopt the condensed notation $p_\theta(z \mid x) = p$. According to the Gaussian conditioning formula, there exists $\mu$ and $\Lambda$ such that 
\[p \sim \mathrm{Normal}\left(\mu, \Lambda^{-1}\right).\]
We consider variational approximations of the form
\[q \sim \mathrm{Normal}\left(\nu, \textrm{diag}(\lambda)^{-1}\right).\]
We wish to characterize the solution $q$ to the following optimization problems:
\begin{align}
\begin{array}{ccc}
    q_{\textrm{RKL}} = \argmin_{q}\kl{q}{p}, & q_{\textrm{FKL}} = \argmin_{q}\kl{p}{q}, & q_\chi = \argmin_{q}\chidiv{p}{q}.
\end{array}
\end{align}

We focus on the setting in which the mean of the variational distribution is correct. This is true for variational Bayes or the general Renyi divergence, as underlined in~\cite{NIPS2016_6208}. Therefore, we further assume $\nu$ can be chosen equal to $\mu$ for simplicity. 

Conveniently, in the bivariate setting we have an analytically tractable inverse formula
\begin{align}
\begin{array}{cc}
    \Lambda = \begin{bmatrix}
    \Lambda_{11}       & \Lambda_{12} \\
    \Lambda_{21}       & \Lambda_{22} 
\end{bmatrix},&
    \Lambda^{-1} = \frac{1}{|\Lambda|} \begin{bmatrix}
    \Lambda_{22}       & -\Lambda_{12} \\
    -\Lambda_{21}       & \Lambda_{11} 
\end{bmatrix}.
\end{array}
\end{align}

We also rely on the expression of the Kullback-Leibler divergence between two multivariate Gaussian distributions of $\mathbb{R}^d$:
\begin{align}
    \kl{\textrm{Normal}\left(\mu, \Sigma_1\right)}{ \textrm{Normal}\left(\mu, \Sigma_2\right)} = \frac{1}{2}\left[\log\frac{|\Sigma_2|}{|\Sigma_1|} - d + \textrm{Tr}(\Sigma_2^{-1}\Sigma_1)\right].
\end{align}

\subsection{Reverse KL}
Using the expression of the KL and the matrix inverse formula, we have that
\begin{align}
    \argmin_q \kl{q}{p} = \argmin_{\lambda_1, \lambda_2} \log \lambda_1\lambda_2 + \frac{\Lambda_{11}}{\lambda_1} + \frac{\Lambda_{22}}{\lambda_2}.
\end{align}
The solution to this optimization problem is 
\begin{align}
\left\{
\begin{array}{ll}
\lambda_1 &= \Lambda_{11} \\
\lambda_2 &= \Lambda_{22} 
\end{array} \right..
\end{align}

\subsection{Forward KL}
From similar calculations,
\begin{align}
    \argmin_q \kl{p}{q} = \argmin_{\lambda_1, \lambda_2} -\log \lambda_1\lambda_2 + \frac{1}{|\Lambda|}\left[\lambda_1\Lambda_{22} + \lambda_2\Lambda_{11}\right].
\end{align}
The solution to this optimization problem is 
\begin{align}
\left\{
\begin{array}{ll}
\lambda_1 &= \Lambda_{11} - \frac{\Lambda_{12}\Lambda_{21}}{\Lambda_{22}} \\
\lambda_2 &= \Lambda_{22} - \frac{\Lambda_{12}\Lambda_{21}}{\Lambda_{11}} 
\end{array} \right..
\end{align}

\subsection{Chi-square divergence}
A closed-form expression of the Renyi divergence for exponential families (and in particular, for multivariate Gaussian distributions) is derived in~\cite{burbea1984convexity}. We could in principle follow the same approach. However, \cite{NIPS2016_6528} derived a similar result, which is exactly the desired quantity for $\alpha = -1$ in Appendix~B of their manuscript. Therefore, we simply report this result:
\begin{align}
\left\{
\begin{array}{ll}
\lambda_1 &= \Lambda_{11}\left[ \frac{3}{2} - \frac{1}{2}\sqrt{1 + \frac{8\Lambda_{12}\Lambda_{21}}{\Lambda_{11}\Lambda_{22}}}\right] \\
\lambda_2 &= \Lambda_{22}\left[ \frac{3}{2} - \frac{1}{2}\sqrt{1 + \frac{8\Lambda_{12}\Lambda_{21}}{\Lambda_{11}\Lambda_{22}}}\right].  
\end{array} \right\}
\end{align}

\subsection{Importance-weighted variational inference} For IWVI, most quantities are not available in closed form. However, the problem is simple and low-dimensional. We use naive Monte Carlo with $10,000$ samples to estimate the IWELBO. The parameters $\lambda_1$ and $\lambda_2$ are the solution to the numerical optimization of the IWELBO (Nelder–Mead method).

\section{Supplemental information for the pPCA experiment}
\label{app:experiment1}
In this appendix, we give more details about the simulation, the construction of the dataset, the model, and the neural network architecture. We also give additional results for a larger number of particles and for benchmarking posterior collapse. 

\subsection{Simulation}

let $p, d \in \mathbb{N}^2, B = [b_1, ..., b_p], C = [c_1, ..., c_p], \nu \in \mathbb{R}^+$. 
We choose our linear system with random matrices
\begin{equation}
\begin{split}
\forall j \leq p, b_j &\sim \textrm{Normal}\left(0, \frac{I_d}{p}\right)  \\
\forall j \leq q, c_j &\sim \textrm{Normal}\left(1, 2\right),  \\
\end{split}
\end{equation}
and define the conditional covariance
\begin{equation}
    \Sigma_{x \mid z} = \nu \times  \text{diag}([c_1^2, \dots, c_p^2]).
\end{equation}

Having drawn these parameters, the generative model is as follows:
\begin{equation}
    \begin{split}
    z &\sim \textrm{Normal}\left(0, I_p\right) \\
    x \mid z &\sim \textrm{Normal}\left(Bz,\Sigma_{x \mid z} \right). 
    \end{split}
    \label{eq:ppca_generative}
\end{equation}
The marginal log-likelihood $p(x)$ is tractable:
\begin{equation}
x \sim \textrm{Normal}\left(0, \Sigma_{x \mid z} + BB^\top \right). 
\end{equation}

The posterior $p(z \mid x)$ is also tractable:
\begin{equation}
\begin{split}
\Sigma^{-1}_{z \mid x} &= I_{p} + A^\top  \Sigma_{x \mid z}^{-1} A  \\
M_{z \mid x} &= \Sigma_{z \mid x} A^\top  \Sigma_{x \mid z}^{-1} \\
z \mid x &\sim \textrm{Normal}\left(M_{z \mid x}x, \Sigma_{z \mid x}\right). 
\end{split}
\label{lin-gauss-posterior}
\end{equation}

The posterior expectation for a toy hypothesis testing $p(z_1 \geq \nu \mid x)$ (with $f: z \mapsto \mathds{1}_{\{z_1 \geq \nu\}}$) is also tractable because this distribution is Gaussian and has a tractable cumulative distribution function.

\subsection{Dataset}
We sample $1000$ datapoints from the generative model (Equation \ref{eq:ppca_generative}) with $p = 10$, $q = 6$, and $\nu=1$. 
We split the data with a ratio of $80\%$ training to $20\%$ testing. 

\subsection{Model details and neural networks architecture}

For every baseline, we partially learned the generative model of Eq.~\eqref{eq:ppca_generative}.
The matrix $B$ was fixed, but the conditional diagonal covariance $\Sigma_{x \mid z}$ weights were set as free parameters during inference.
Neural networks with one hidden layer (size $128$), using ReLu activations, parameterized the encoded variational distributions.

Each model was trained for $100$ epochs, and optimization was performed using the Adam optimizer (learning rate of 0.01, batch size 128).

\subsection{Additional results}

We compare PSIS levels for the pPCA dataset for each model (figure \ref{fig:mnist_cross_other}).
For most models (IWAE, WW, and $\chi$), the VAE variational distribution provides poor importance-weighted estimates.
The PSIS exceeds $0.7$ for those combinations, hinting that associated samples may be unreliable.
Most other combinations show acceptable PSIS levels, with the proposals from $\rchi$ and MIS performing best.

\begin{figure}[H]
    \centering
        \includegraphics[width=0.3\textwidth]{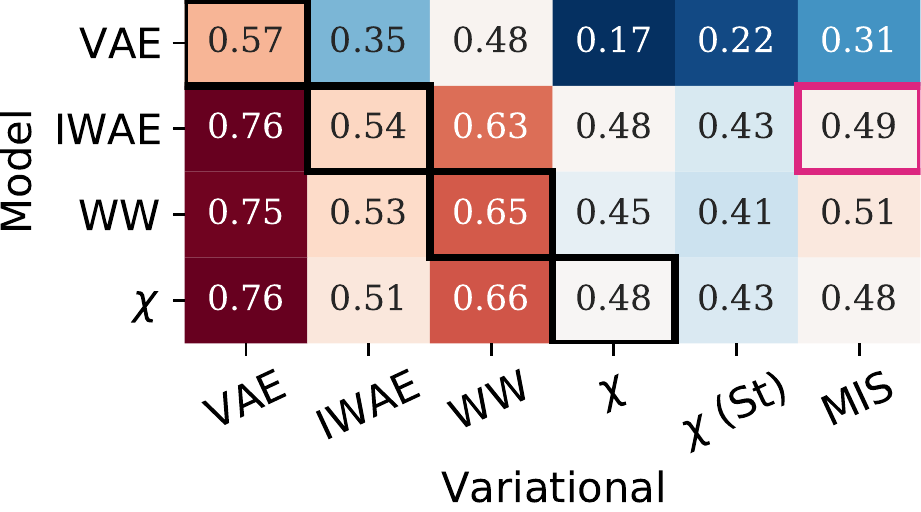}
    \caption{PSIS for pPCA. Each row corresponds to an objective function for fitting the model parameters and each column corresponds to an objective function for fitting the variational parameters.}
    \label{fig:mnist_cross_other}
    \vspace{-0.3cm}
\end{figure}

\subsection{Results with an increased number of particles}

\begin{figure}
\begin{minipage}{0.55\textwidth}
\centering
\captionsetup{type=table}
\begin{small}
\begin{sc}
\caption{Results on the pPCA simulated data. MAE refers to the mean absolute error in posterior expectation estimation.}
\label{table:ppca_k200}
\begin{tabular}{lcccc}
\toprule
& \textbf{VAE} & \textbf{IWAE} & \textbf{WW} & \textbf{$\rchi$-VAE}\\
\midrule
$\log p_\theta(X)$  & -17.22 & -16.93 & \textbf{-16.92} & -17.28\\
IWELBO  & -17.22 & -16.93 & \textbf{-16.92} & -17.29\\
\midrule
$||A||_2$  & 1.62 & \textbf{0.96} & 1.16 & 1.00\\
PSIS  & 0.56 & \textbf{0.07} & 0.49 & 0.98\\
\midrule
MAE  & 0.062 & 0.028 & \textbf{0.021} & 0.073\\
\bottomrule
\end{tabular}
\end{sc}
\end{small}

\end{minipage}
\hfill
\begin{minipage}{0.4\textwidth}
\centering
\includegraphics[width=0.8\textwidth]{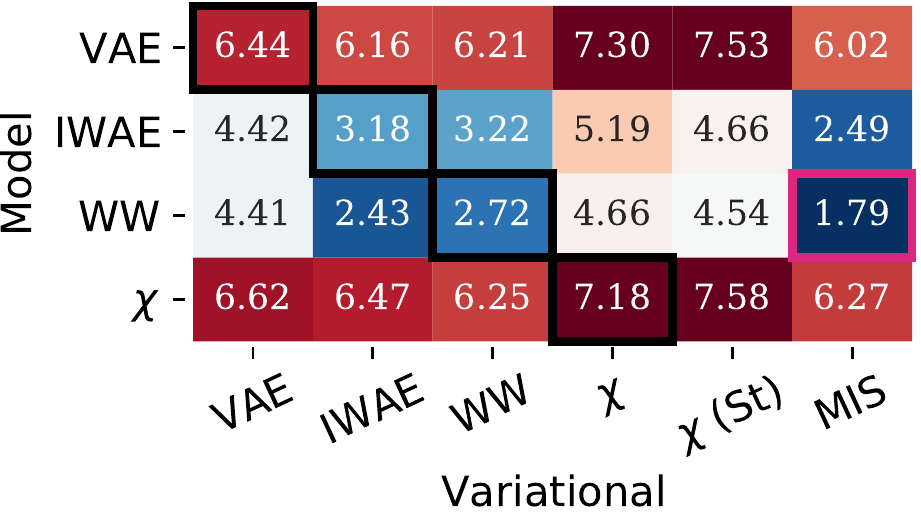}
\caption{MAE ($\times 100$) for pPCA.
Each row corresponds to an objective function for fitting the model parameters and each column corresponds to an objective function for fitting the variational parameters.}
\label{fig:ppca_k200}
\end{minipage}
\vspace{-0.3cm}
\end{figure}

We also benchmark the different algorithms for an increased number of particles (Table \ref{table:ppca_k200}).
In this setup, we can observe that the IWAE model performance worsens in terms of held-out likelihood with a high number of particles, underlining a well-known behavior of this model~\cite{rainforth2018tighter}.
Conversely, increasing the number of particles is more beneficial to WW than to IWAE. 
WW learns the best generative model (in terms of held-out likelihood) and reaches lower mean absolute errors than IWAE. 
Intriguingly, the performance of the $\rchi$-VAE drops significantly on all metrics in this setup.
As in the other experiments, our three-step approach minimizes the MAE, among all generative model/variational distribution pairings (Figure \ref{fig:ppca_k200}).

\subsection{Benchmarking for posterior collapse methods}

Posterior collapse is an established issue of VAE training in which the variational network does not depend on the data instance.
Currently, there are two different explanations for this behavior. 
In some research lines, it is assumed to be a specificity of the inference procedure\cite{fu2019cyclical, he2018lagging}.
In others, it is thought to be caused by a deficient model~\cite{NIPS2019_9138}.  
The second interpretation is beyond the scope of our manuscript.
To measure the impact of posterior collapse, we included cyclic KL annealing ~\cite{fu2019cyclical} and lagging inference network~\cite{he2018lagging} baselines to the pPCA experiment. 
We also considered constant KL annealing, which did not improve performance over cyclical annealing.

These methods improve the held-out log-likelihood and MAE of the VAE baseline (Table \ref{table:ppca_extended}), hinting that posterior collapse alleviation can improve decision-making.
However, even the best performing method (lagging inference networks) shows slight improvement over the VAE baseline ($2\%$ in terms of held-out likelihood) and does not reach the other baseline performances.
We leave extended studies of posterior collapse effects to future work.

\begin{figure}[h]
\centering
\captionsetup{type=table}
\begin{small}
\begin{sc}
\caption{Extended results on the pPCA simulated data.}
\label{table:ppca_extended}
\begin{tabular}{lcccccc}
\toprule
& \textbf{VAE} & \textbf{Agg}& \textbf{Cyclic}&\textbf{IWAE} & \textbf{WW} & \textbf{$\rchi$-VAE}\\
$\log p_\theta(X)$  &
-17.65 & -17.13 & -17.20 & \textbf{-16.91} & -16.93 & -16.92\\
IWELBO  & -17.66 & -17.14 & -17.20 & \textbf{-16.92} & -16.96 & \textbf{-16.92}\\
\midrule
$\norm{A}_2$  & 1.69 & 1.47 & 1.68 & 1.30 & 2.32 & \textbf{1.13}\\
PSIS  & 0.54 & 0.55 & 0.58 & 0.53 & 0.66 & \textbf{0.47}\\
\midrule
MAE  & 1.03 & 0.057 & 0.050 & 0.032 & 0.043 & \textbf{0.030}\\
\bottomrule
\end{tabular}
\end{sc}
\end{small}

\end{figure}

\section{Supplemental information for the MNIST experiment}
\label{app:experiment2}

\subsection{Dataset}
We used the MNIST dataset~\cite{mnist}, and split the data using a $50\%$ training to $50\%$ test ratio. 

\subsection{Model details and neural networks architecture}

For efficiency considerations, the variational distribution parameters of $q_\phi(z \mid x)$ were parameterized using a small convolutional neural network (3 layers of size-3 kernels), followed by two fully-connected layers.
The parameters of the distributions $q_\phi(c \mid z), q_\phi(u \mid z, c), p_\theta(x \mid z), p_\theta(z \mid c, u),$ and $p_\theta(c)p_\theta(u)$ were all encoded by fully-connected neural networks (one hidden layer of size 128).
We used SELU non-linearities~\cite{klambauer2017self} and a dropout (rate 0.1) between all hidden layers.

All models were trained for 100 epochs using the Adam optimizer (with a learning rate of $0.001$ and a batch size of $512$).

\subsection{Additional results}

All models show relatively similar accuracy levels (Figure \ref{fig:Xmnist_cross_other}).
The three-step procedure applied to the best generative model (IWAE) provides the best levels of accuracy.

\begin{figure}[H]
    \centering
        \includegraphics[width=0.3\textwidth]{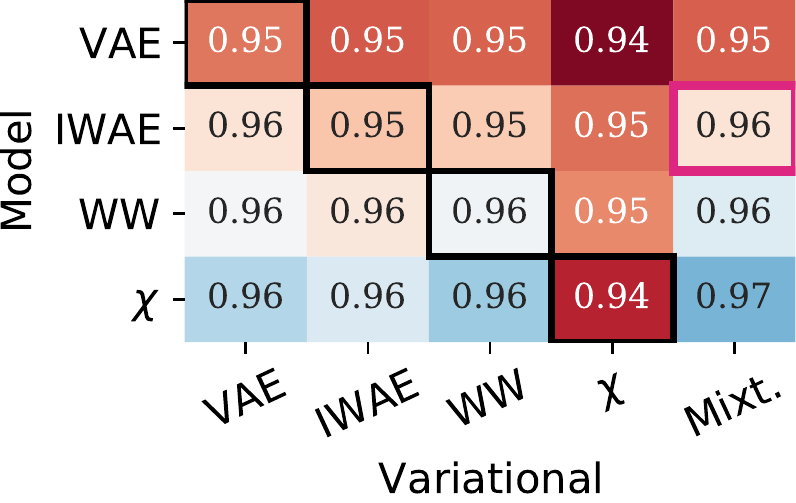}
    \caption{
    Accuracy for MNIST.
Each row corresponds to an objective function for fitting the model parameters and each column corresponds to an objective function for fitting the variational parameters.}
    \label{fig:Xmnist_cross_other}
    \vspace{-0.3cm}
\end{figure}

\subsection{Estimation of posterior expectations for the M1+M2 model}

Here we derive the two estimators for estimating $p_\theta(c \mid x)$ in the M1+M2 model. First, we remind the reader that the generative model is
\begin{align}
\label{eq:m1m2_gen_app}
    p_\theta(x, z, c, u) &= p_\theta(x \mid z)p_\theta(z \mid c, u)p_\theta(c)p_\theta(u)
\end{align}
and that the variational distribution factorizes as
\begin{align}
q_\phi(z, c, u \mid x) &= q_\phi(z \mid x)q_\phi(c \mid z)q_\phi(u \mid z, c).
\end{align}

\paragraph{Plugin approach}
For the plugin approach, we compute $q_\phi(c \mid x)$ as
\begin{align}
    q_\phi(c \mid x) &= \iint_{z, u} q_\phi(c, u, z \mid x)du\,dz \\ 
    &= \iint_{z, u} q_\phi(u \mid c, z)q_\phi(c, z \mid x)du\,dz \\
    &= \iint_{z} q_\phi(c \mid z) q_\phi(z \mid x)dz,
\end{align}
where the last integral is estimated with naive Monte Carlo. 

\paragraph{SNIS approach}
We obtain $p_\theta(c, x)$ via marginalization of the latent variables $z, u$:
\begin{align}
    p_\theta(c, x)
    &= 
    \iint
    p_\theta(x, u, c, z) dz\,du.
\end{align}
We may estimate this probability, for a fixed $c$, using $q_\phi(z, u \mid x, c)$ as a proposal for importance sampling:
\begin{align}
p_\theta(c, x)
    &=
    \mathbb{E}
    _{q_\phi(z, u \mid x, c)}
    \left[
        \frac{p_\theta(x, u, c, z)}{q_\phi(z, u \mid x, c)}
    \right].
\end{align}
Then, the estimates for $p_\theta(c, x)$ may be normalized by their sum for all labels (equal to $p_\theta(x)$) to recover $p_\theta(c \mid x)$. 

Interestingly, this estimator does not make use of the classifier $q_\phi(c \mid z)$, so we expect it to possibly have lower performance than the plugin estimator. Indeed, the $q_\phi(c \mid z)$ is fit with a classification loss based on the labeled data points. 

\section{Analysis of alternate divergences for the M1+M2 model}
\label{app:m1m2}
We have the following pathological behavior, similar to that presented in the factor analysis instance. This pathological behavior is exacerbated when the M1+M2 model is fitted with a composite loss as in Equation~9 of~\cite{KingmaRMW14}. Indeed, neural networks are known to have poorly calibrated uncertainties~\cite{pmlr-v70-guo17a}.
\begin{prop}
\label{prop:M1M2}
Consider the model defined in Eq.~\eqref{eq:m1m2_gen}. Assume that posterior inference is exact for latent variables $u$ and $z$, such that $q_\phi(z \mid x)q_\phi(u \mid c, z) = p_\theta(z, u \mid c, x)$. Further, assume that for a fixed $z \in \mathcal{Z}$, $p_\theta(c \mid z)$ has complete support. Then, as $\mathbb{E}_{q_\phi(c \mid z)}\log q_\phi(c \mid z) \rightarrow 0$, it follows that
\begin{enumerate}
    \item $\kl{q_\phi(c, z, u \mid x)}{p_\theta(c, z, u \mid x)}$ is bounded;
    \item $\kl{p_\theta(c, z, u \mid x)}{q_\phi(c, z, u \mid x)}$ diverges; and
    \item $\chidiv{p_\theta(c, z, u \mid x)}{q_\phi(c, z, u \mid x)}$ diverges.
\end{enumerate}
\end{prop}
\begin{proof}
The proof mainly consists of decomposing the divergences. The posterior for unlabeled samples factorizes as
\begin{align}
    p_\theta(c, u, z \mid x) = p_\theta(c \mid z) p_\theta(z, u \mid x, c).
\end{align}
From this, expressions of the other divergences follow from the semi-exact inference hypothesis. Remarkably, all three divergences can be decomposed into similar forms. Also, the expression of the divergences can be written in closed form (recall that $c$ is discrete) and as a function of $\lambda$.

\textit{Reverse-KL.} In this case, the Kullback-Leibler divergence can be written as
\begin{align}
    \kl{q_\phi(c, z, u \mid x)}{p_\theta(c, z, u \mid x)} &= \mathbb{E}_{q_\phi(c \mid z)} \kl{q_\phi(z, u \mid x, c)}{p_\theta(z, u\mid x, c)} \\
    &+\mathbb{E}_{q_\phi(z \mid x)} \kl{q_\phi(c \mid z)}{p_\theta(c \mid z)},
\end{align}
which further simplifies to 
\begin{align*}
    \kl{q_\phi(c, z, u \mid x)}{p_\theta(c, z, u \mid x)} &= \mathbb{E}_{p_\theta(z \mid x)} \kl{q_\phi(c \mid z)}{p_\theta(c \mid z)}.
\end{align*}
This last equation can be rewritten as a constant plus the differential entropy of $q_\phi(c \mid z)$, which is bounded by $\log C$ in absolute value.

\textit{Forward-KL.} Similarly, we have that
\begin{align*}
    \kl{p_\theta(c, z, u \mid x)}{q_\phi(c, z, u \mid x)} &= \mathbb{E}_{p_\theta(c \mid z)} \kl{p_\theta(z, u\mid x, c)}{q_\phi(z, u \mid x, c)} \\
    &+\mathbb{E}_{p_\theta(z \mid x)} \kl{p_\theta(c \mid z)}{q_\phi(c \mid z)},
\end{align*}
which also further simplifies to
\begin{align*}
    \kl{p_\theta(c, z, u \mid x)}{q_\phi(c, z, u \mid x)} &= \mathbb{E}_{p_\theta(z \mid x)} \kl{p_\theta(c \mid z)}{q_\phi(c \mid z)} \\
    &= \mathbb{E}_{p_\theta(z \mid x)} \sum_{c=1}^C p_\theta(c \mid z) \log \frac{p_\theta(c \mid z)}{q_\phi(c \mid z)}.
\end{align*}
This last equation includes terms in $p_\theta(c \mid z)\log q_\phi(c \mid z)$, which are unbounded whenever $q_\phi(c \mid z)$ is zero but $p_\theta(c \mid z)$ is not.

\textit{Chi-square.} Finally, for this divergence, we have the decomposition
\begin{align*}
    \chidiv{p_\theta(c, z, u \mid x)}{q_\phi(c, z, u \mid x)} &= \mathbb{E}_{q_\phi(z, c, u \mid x)} \frac{p^2_\theta(z, u\mid x, c)p^2_\theta(c \mid z)}{q^2_\phi(z, u \mid x, c)q^2_\phi(c \mid z)},
\end{align*}
which in this case simplifies to
\begin{align*}
    \chidiv{p_\theta(c, z, u \mid x)}{q_\phi(c, z, u \mid x)} = \mathbb{E}_{p_\theta(z \mid x)} \chidiv{p_\theta(c \mid z)}{q_\phi(c \mid z)}.
\end{align*}
Similarly, the last equation includes terms in $\nicefrac{p^2_\theta(c \mid z)}{q_\phi(c \mid z)}$, which are unbounded whenever $q_\phi(c \mid z)$ is zero but $p_\theta(c \mid z)$ is not.
\end{proof}

\section{Supplemental information for the single-cell experiment}
\label{app:experiment3}

\subsection{Dataset}

Let $N$ and $G$ denote the number of cells and genes of the dataset, respectively.
We simulated scRNA counts from two cell-states $a$ and $b$, each following a Poisson-lognormal distribution with respective means $\mu_{ag}$ and $\mu_{bg}$ for $g \leq G$, and sharing covariance $\Sigma$.
For each cell $n$, the cell-state $c_n$ is modelled as a categorical distribution of parameter $p$.
The underlying means follow log-normal distributions
\begin{align*}
    h_{ng} \sim \text{LogNormal}(\mu_{c_n}, \Sigma).
\end{align*}
Counts $x_{ng}$ for cell $n$ and gene $g$ are assume to have Poisson noise
\begin{align*}
    x_{ng} \sim \text{Poisson}(h_{ng}).
\end{align*}

We now clarify how the log-normal parameters were constructed.
Both populations shared the same covariance structure
\begin{align*}
    \Sigma = 
    (0.5 + u) I_g
    + 
    2 a a^T,
    ~~\text{where}~~
    \begin{cases}
    a \sim \mathcal{U}
    ((-1, 1)^g)
     \\ 
    u \sim \mathcal{U}
    ((-0.25, 0.25)^g).
    \end{cases}
\end{align*}

The ground-truth LFC  values $\Delta_g$ between the two cell states, $a$ and $b$, were randomly sampled in the following fashion.
We first randomly assign a differential expression status to each gene.
It can correspond to similar expression, up-regulation, or down-regulation between the two states for the gene.
Conditioned to this status, the LFCs were drawn from Gaussian distributions respectively centered on $0, -1$, and $1$ and of standard deviation $\sigma = 0.16$.

Finally, gene expression means for population $a$ were sampled uniformly on $(10, 100)$
populations $b$ obtained as
\begin{align*}
    \mu_{b} =2^{\Delta_g} \mu_a.
\end{align*}

In our experiments, we used $N=1000$ and $G=100$, and followed a $80\%-20\%$ train-test split ratio.

\subsection{Model details and neural networks architecture}

Here we introduce a variant of scVI as a generative model of cellular expression counts. For more information about scVI, please refer to the original publication~\cite{Lopez292037}.

\paragraph{Brief background on scVI} Latent variable $z_n \sim \text{Normal}(0, I_d)$
represents the biological state of cell $n$. Latent variable $l_n \sim \text{LogNormal}(\mu_l, \sigma_l^2)$ represents the library size (a technical factor accounting for sampling noise in scRNA-seq experiments). Let $f_w$ be a neural network. For each gene $g$, expression count $x_{ng}$ follows a zero-inflated negative binomial distribution whose negative binomial mean is the product of the library size $l_n$ and normalized mean $h_{ng} = f_w(z_n)$. The normalized mean $h_{ng}$ is therefore deterministic conditional on $z_n$; it will have uncertainty in the posterior due to $z_n$. The measure $p_\theta(h_{ng} \mid x_n)$ denotes the push-forward of $p_\theta(z_n \mid x_n)$ through the $g$-th output neuron of neural net $f_w$. scVI therefore models the distribution $p_\theta(x)$.

\paragraph{Differences introduced} In our experiments, the importance sampling weights for all inference mechanisms had high values of the PSIS diagnostic for the original scVI model.
Although the FDR control was more efficient with alternative divergences, our proposal distributions were poor. The posterior variance for latent variable $l_n$ could reach high values, leading to numerical instabilities for the importance sampling weights (at least on this dataset). To work around the problem, we removed the prior on $l_n$ and learned a generative model for the conditional distribution $p_\theta(x_n \mid l_n)$ using the number of transcripts in cell $n$ as a point estimate of $l_n$.

\subsection{Additional results}

\begin{figure}[H]
    \centering
    \includegraphics[width=0.65\textwidth]{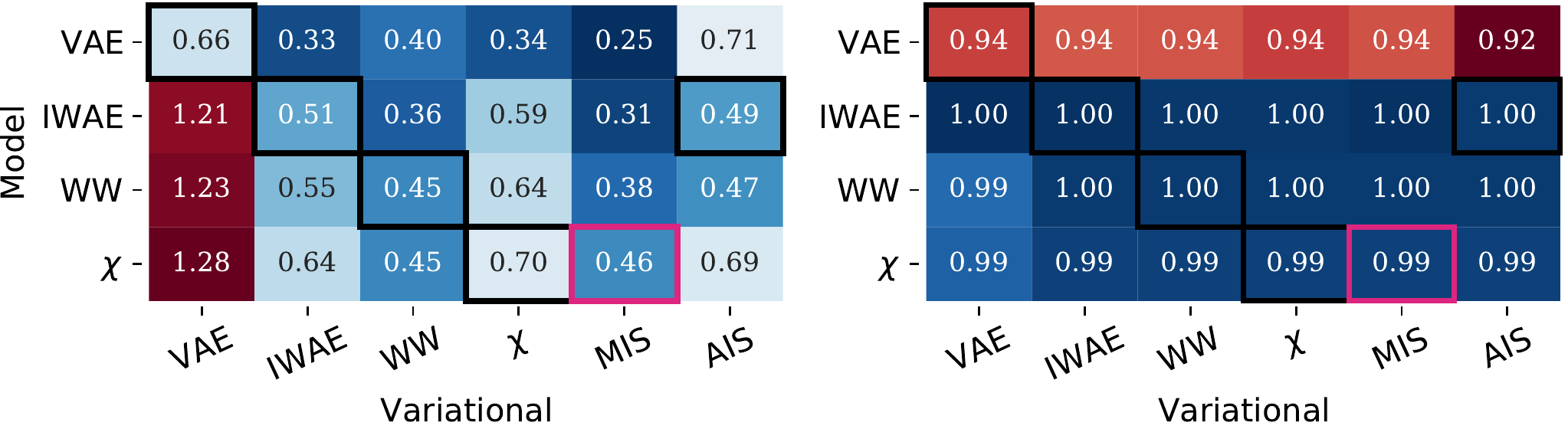}
    \caption{
    PSIS (\textit{left}) and PRAUC (\textit{right}) for scVI. Each row corresponds to an objective function for fitting the model parameters and each column corresponds to an objective function for fitting the variational parameters.}
    \label{fig:XscVI-cross-other}
    \vspace{-0.3cm}
\end{figure}

We emphasize that the PSIS metric does not provide a complete picture for selecting a decent model/variational distribution combination.
On the differential expression task, most combinations using VAEs as generative models offer appealing PSIS values (Figure \ref{fig:XscVI-cross-other}). However,  these combinations offer deceiving gene rankings, as hinted by their PRAUC ($AUC=0.94$).  
In addition, the variational distributions trained using the classical ELBO used in combination with IWAE, WW, or $\chi$ are inadequate for decision-making.
These blends reach inadmissible levels of PSIS.

To assess the potential of the different models for detecting differential expression, we compare the FDR evolution with the posterior expected FDR of the gene rankings obtained by each model (Figure \ref{fig:YscVI-cross-other}).
The match between these quantities for IWAE and $\chi$ hints that they constitute sturdy approaches for differential expression tasks, while the traditional VAE fails to estimate FDR reliably.

\begin{figure}[H]
    \centering
    \includegraphics[width=0.7\textwidth]{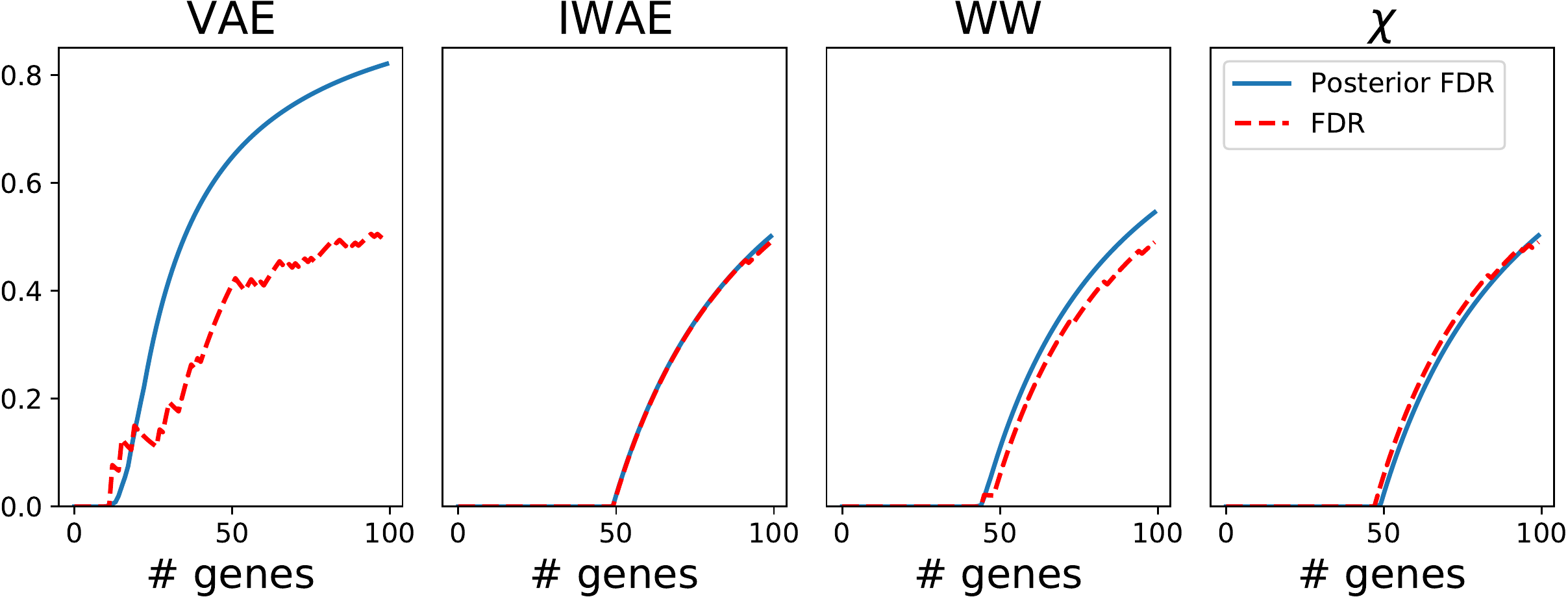}
    \caption{
    Posterior expected FDR (blue) and ground-truth FDR (red) for the decision rule that selects the genes with the highest DE probability.
    }
    \label{fig:YscVI-cross-other}
    \vspace{-0.3cm}
\end{figure}

\end{document}